\documentclass[10pt,journal,compsoc]{IEEEtran}
\usepackage[nocompress]{cite}
\usepackage[pdftex]{graphicx}
\usepackage{balance}
\graphicspath{{figures/}}
\DeclareGraphicsExtensions{.pdf,.jpeg,.png}
\usepackage{amsmath,bm}
\usepackage{amsthm}
\usepackage{bbm}
\usepackage{algorithmic}
\usepackage{array}
\usepackage{url}
\usepackage[x11names]{xcolor}
\usepackage{tabulary,xfrac,enumitem}
\usepackage{booktabs}       
\usepackage{amsfonts}       
\usepackage{nicefrac}       
\usepackage{microtype}      
\usepackage{amssymb}
\usepackage{amsmath}
\usepackage{caption, subcaption}
\usepackage{multirow}
\usepackage{textcomp}
\usepackage{makecell}
\usepackage{colortbl}
\usepackage{pifont}
\usepackage{xspace}
\usepackage{subcaption}
\usepackage[ruled,linesnumbered]{algorithm2e}
\usepackage[pagebackref=true,breaklinks=true,colorlinks,citecolor=citecolor,bookmarks=false]{hyperref}
\definecolor{citecolor}{RGB}{34,139,34}

\begin{document}
\title{Modeling Uncertain Feature Representation for Domain Generalization}

\author{Xiaotong Li, Zixuan Hu, Jun Liu, Yixiao Ge, Yongxing Dai, Ling-Yu Duan
\IEEEcompsocitemizethanks{\IEEEcompsocthanksitem X. Li, Z. Hu, and Y. Dai are with Peking University, Beijing, China. Email: lixiaotong@stu.pku.edu.cn, hzxuan@pku.edu.cn, and yongxingdai@pku.edu.cn.
\IEEEcompsocthanksitem J. Liu is with the Information Systems Technology and Design Pillar, Singapore University of Technology and Design, Singapore. Email: jun\_liu@sutd.edu.sg.
\IEEEcompsocthanksitem Y. Ge is with the ARC Lab, Tencent PCG, China. Email: yixiaoge@tencent.com.
\IEEEcompsocthanksitem L. Duan is with the Peking University, Beijing, China, and also with Peng Cheng Laboratory, Shenzhen, China. Email: lingyu@pku.edu.cn.
}
\thanks{Corresponding author: Ling-Yu Duan}}


\markboth{Work on progress}%
{Shell \MakeLowercase{\textit{et al.}}: Bare Demo of IEEEtran.cls for Computer Society Journals}
\IEEEtitleabstractindextext{%
\begin{abstract}
Though deep neural networks have achieved impressive success on various vision tasks, obvious performance degradation still exists when models are tested in out-of-distribution scenarios. In addressing this limitation, we ponder that the feature statistics (mean and standard deviation), which carry the domain characteristics of the training data, can be properly manipulated to improve the generalization ability of deep learning models. Existing methods commonly consider feature statistics as deterministic values measured from the learned features and do not explicitly model the uncertain statistics discrepancy caused by potential domain shifts during testing.
In this paper, we improve the network generalization ability by modeling domain shifts with uncertainty (DSU), \textit{i.e.}, characterizing the feature statistics as uncertain distributions during training.
Specifically, we hypothesize that the feature statistic, after considering the potential uncertainties, follows a multivariate Gaussian distribution. Consequently, each feature statistic is no longer a deterministic value, but a probabilistic point with diverse distribution possibilities.
With the uncertain feature statistics, the model can be trained to well alleviate the domain perturbations. 
During inference, we propose an instance-wise adaptation strategy that can adaptively deal with the unforeseeable shift and further enhance the generalization ability of the trained model with negligible additional cost. We also conduct theoretical analysis on the aspects of generalization error bound and the implicit regularization effect, showing the efficacy of our method. Extensive experiments demonstrate that our method consistently improves the network generalization ability on multiple vision tasks, including image classification, semantic segmentation, instance retrieval, and pose estimation. Our methods are simple yet effective and can be readily integrated into networks without additional trainable parameters or loss constraints. Code will be released in \href{https://github.com/lixiaotong97/DSU}{https://github.com/lixiaotong97/DSU}.
\end{abstract}

\begin{IEEEkeywords}
Domain generalization, out-of-distribution, robustness, uncertainty
\end{IEEEkeywords}}

\maketitle
\IEEEdisplaynontitleabstractindextext
\IEEEpeerreviewmaketitle

\IEEEraisesectionheading{\section{Introduction}\label{sec:introduction}}
\IEEEPARstart{D}{eep} neural networks have demonstrated remarkable success in computer vision, but with a severe reliance on the assumption that the training and testing domains follow an independent and identical distribution \cite{ben2010theory,iidrisk}.
This assumption, however, does not hold in many real-world applications.
For instance, when employing segmentation models trained on sunny days for rainy and foggy environments \cite{RobustNet_2021_CVPR}, or recognizing art paintings with models that trained on photographs \cite{pacs_2017_ICCV}, inevitable performance drop can often be observed in such out-of-distribution deployment scenarios. 
Therefore, the problem of domain generalization is proposed to improve the robustness of the network on various unseen testing domains and becomes quite important for the development of vision models.

Domain characteristics primarily refer to the information that is more specific to the individual domains but less relevant to the task objectives, such as the photo style and capturing environment information in object recognition. According to works \cite{adain,li2021feature,mixstyle}, feature statistics (mean and standard deviation), as the moments of the learned features, carry informative domain characteristics of the training data. 
Consequently, domains with different data distributions generally have inconsistent feature statistics \cite{nips20calibration,nips19trans,rbn}.
Most deep learning methods follow Empirical Risk Minimization principle (ERM) \cite{erm} to minimize their average error over the training data \cite{cuipeng2021towards}.
Despite the satisfactory performance on the training domain, these methods do not explicitly consider the uncertain statistics discrepancy caused by potential domain shifts during testing. 
As a result, the trained models tend to overfit the training domain and show vulnerability to the statistic changes at testing time, substantially limiting the generalization ability of the learned representations.

Due to the diverse underlying possibilities of test domains and their training-time unforeseeablity, they may bring uncertain statistics shifts with different potential directions and intensities in space compared to the training domain (as shown in Figure \ref{fig:extra}), implying the uncertain nature of domain shifts. 
Considering such ``uncertainty'' of potential domain shifts, synthesizing novel feature statistics variants to model diverse domain shifts can improve the robustness of the trained network to different testing distributions.
Towards this end, we introduce a novel probabilistic method to improve the network generalization ability by properly \textit{modeling} \textbf{\emph{D}}\textit{omain}  \textbf{\emph{S}}\textit{hifts with}  \textbf{\emph{U}}\textit{ncertainty} (DSU), \textit{i.e.}, characterizing the feature statistics as uncertain distributions to model the diverse domain shifts.

In the proposed method, instead of treating each feature statistic as a deterministic point measured from the feature, we hypothesize that the feature statistic, after considering potential uncertainties, follows a multi-variate Gaussian distribution.
The distribution ``center'' is set as each feature's original statistic value, and the distribution ``scope'' represents the variant intensity considering underlying domain shifts. 
Uncertainty estimation is adopted here to depict the distribution ``scope'' of probabilistic feature statistics. Specifically, we estimate the distribution ``scope'' based on the variances of the mini-batch statistics in an efficient non-parametric manner. 
Subsequently, feature statistics variants are randomly sampled from the estimated Gaussian distribution and then used to replace the original deterministic values for modeling diverse domain shifts, as illustrated in Figure \ref{fig:gaussian}. 

The proposed uncertainty modeling approach during training can well utilize the observations of training domains to generate diverse and reasonable feature statistics perturbations, thus the model can be trained to effectively alleviate the domain shifts. After training, the model becomes static and will encounter kinds of test data. To further enhance the generalization ability to unforecastable domain shifts, we provide an inference-time adaptation strategy to adaptively calibrate the instance-wise feature statistics of test samples, which can provide better performances with negligible additional cost.

In this paper, we also
conduct a detailed analysis about the efficacy of our method from theoretical perspectives. On the one hand, for analyzing the effect of manipulations on feature statistics, we make an extension to obtain a \textit{generalization error bound} on the feature space, from which we can obtain an overall illustration about how our method acts on improving network generalization ability. On the other hand, we provide an analysis to show that DSU brings an \textit{implicit regularization effect} during the training dynamic process, which implicitly aligns the feature statistics of multiple source domains and can help the model encode better domain-invariant features.

A comprehensive evaluation is conducted to verify the effectiveness of our method to improve generalization ability on different kinds of vision tasks, \textit{e.g.,} \textit{multi-domain classification}, \textit{instance retrieval}, \textit{semantic segmentation}, and \textit{pose estimation}, showing that introducing uncertainty to feature statistics can well improve models' generalization against domain shifts and the performances can be further improved by incorporating the inference-time adaptation strategy.
Our method does not change the inherent model design and can be readily integrated into existing networks without bringing additional trainable parameters or loss constraints, which is simple yet effective to alleviate performance drops caused by domain shifts.

\begin{figure*}[h]
\begin{center}
\includegraphics[width=1.0\textwidth]{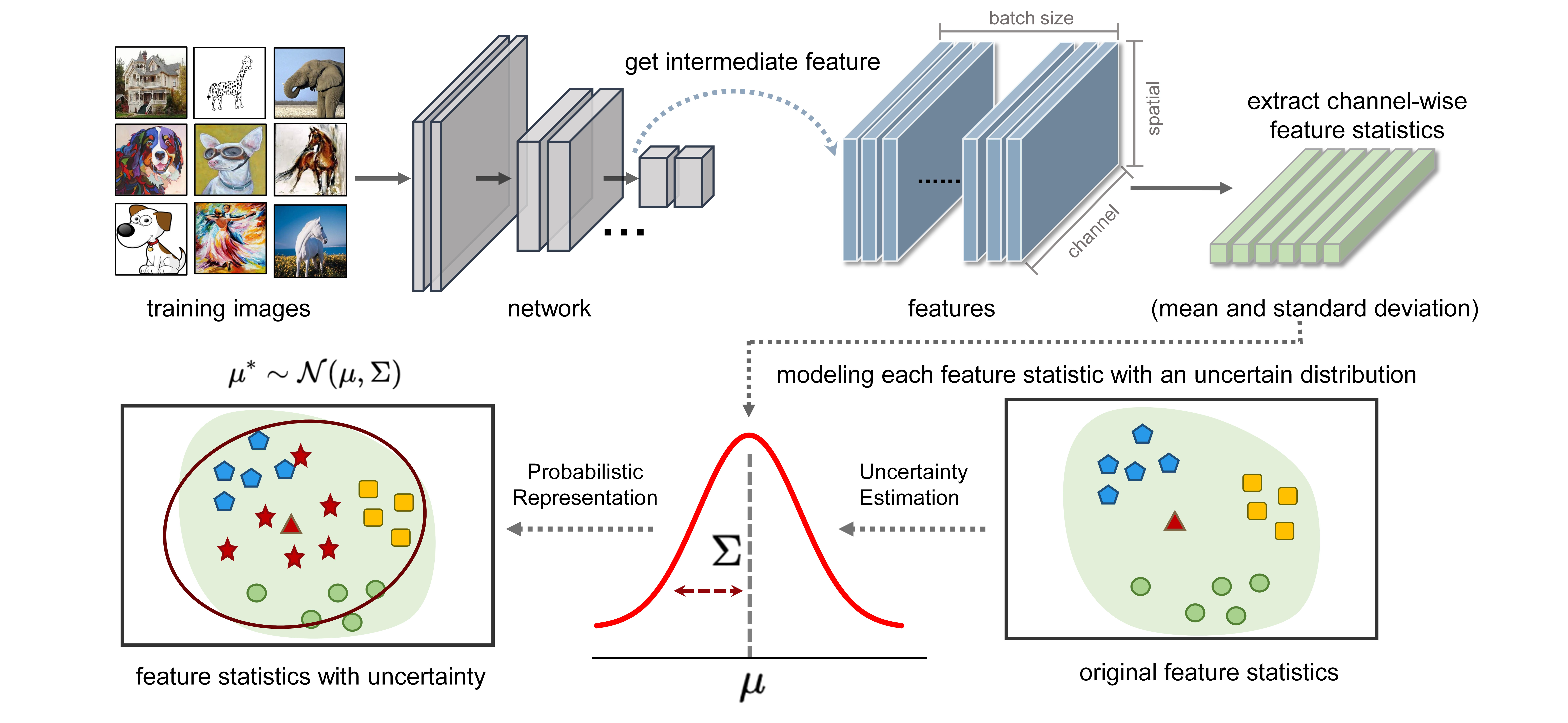}
\end{center}
\caption{Illustration of the proposed method DSU. Feature statistic is assumed to follow a multi-variate Gaussian distribution during training. When passed through this module, the new feature statistics randomly drawn from the corresponding distribution will replace the original ones to model the diverse domain shifts.}
\label{fig:gaussian}
\end{figure*}

Our main contributions can be summarized as follows:
\begin{itemize}
    \item[(1)] We propose to model the uncertainty of the domain shift during training, by considering each feature statistic following a multi-variate Gaussian distribution to model the diverse potential domain shifts.
    \item[(2)] We devise an inference-time adaptation strategy to further enhance the model generalization ability to deal with various unforecastable domain shifts.
    \item[(3)] We conduct a theoretical analysis of the feature space generalization error bound and implicit regularization effect to obtain an inclusive understanding of our method.
    \item[(4)] Comprehensive experiments on a wide range of vision tasks demonstrate the superiority and effectiveness of the proposed method. 
    
\end{itemize}

This work is an extension of our previous conference paper \cite{dsu}. The contributions of this paper over \cite{dsu} are as following several aspects. \textbf{Method:} other than the training-time uncertainty modeling introduced in \cite{dsu}, we further enhance it with instance adaptation during inference, jointly establishing a system-level solution for domain generalization problem. \textbf{Theoretical analysis:} to incisively show its working mechanism, we theoretically analyze our method \textit{w.r.t.} the generalization error bound and implicit regularization effect. \textbf{Empirical study:} we demonstrate the effectiveness of our method and its generalization ability by conducting more comprehensive experiments (\textit{e.g.}, pose estimation) and providing different kinds of detailed quantitative analysis.

\section{Related Work}

\subsection{Domain Generalization}
Domain generalization (DG) has been attracting increasing attention in the past few years, which aims to achieve out-of-distribution generalization on unseen target domains using only single or multiple source domain data for training \cite{NIPS2011_b571ecea}. Research on addressing this problem has been extensively conducted in the literature \cite{zhou2021domain,wangjingdong2021generalizing,cuipeng2021towards}. Here some studies that are more related to our work are introduced below.

\textbf{Data Augmentation}: Data augmentation is an effective manner for improving generalization ability and relieving models from overfitting in training domains. 
Most augmentation methods adopt various transformations at the image-level, such as AugMix \cite{2020augmix} and CutMix \cite{cutmix}. Besides using handcraft transformations,
mixup \cite{zhang2018mixup} trains the model by using pair-wise linearly interpolated samples in both the image and label spaces. Manifold Mixup \cite{manifoldmixup} further adopts
this linear interpolation from image level to feature level. 
Some recent works extend the above transformations to feature statistics for improving model generalization. AdaIN \cite{adain} proposes an instance style transformation with the feature statistics of a reference image. MixStyle \cite{mixstyle} adopts linear interpolation on feature statistics of two instances to generate synthesized samples. pAdaIn \cite{Nuriel_2021_CVPR} swaps statistics between the samples applied with a random permutation of the batch. EFDM \cite{efdm} implicitly performs the exact histogram matching on features to achieve style transfer. Different from them, our method, not relying on a specific reference sample, is based on the adaptative distribution that can produce not only linear changes but diverse variants with more possibilities.

\textbf{Invariant Representation Learning}: 
The main idea of invariant representation learning is to enable models to learn features that are invariant to domain shifts. Domain alignment-based approaches \cite{Li_2018_ECCV, Adversarial} learn invariant features by minimizing the distances between different distributions. Instead of enforcing the entire features to be invariant, disentangled feature learning approaches \cite{2020EccvDMG,pmlr-v119-lowrank} decouple the features into domain-specific and domain-invariant parts and learn their representations simultaneously. In addition, normalization-based methods, \cite{pan2018IBN_Net,RobustNet_2021_CVPR} adopt instance/batch normalization layers to remove the style information to obtain invariant representations. In contrast, DSU is not designed to normalize the features but to generate probabilistic feature statistics for modeling the diverse domain shifts, which is independent from model's inherent normalization layers and can be flexibly inserted into different positions. We also note that DSU is adopted during training and can be discarded during testing without changing the model design. 

\textbf{Learning Strategies}:
There are also some effective learning strategies that can be leveraged to improve generalization ability. 
Ensemble learning is an effective technique for boosting model performance. The ensemble predictions using a collection of diverse models \cite{dael} or modules \cite{DSON} can be adopted to improve generalization and robustness. Meta-learning-based methods \cite{maml,MLDG, Dai_2021_CVPR} learn to simulate the domain shifts following an episode training paradigm. Besides, self-challenging methods, such as RSC \cite{huangRSC2020}, force the model to learn a general representation by discarding dominant features activated on the training data.

\subsection{Uncertainty in Deep Learning}
Uncertainty capturing the ``noise'' and ``randomness'' inherent in the data has received increasing attention in deep representation learning. Variational Auto-encoder \cite{vae}, as an important method for learning generative models, can be regarded as a method to model the data uncertainty in the hidden space. Dropout \cite{dropout}, which is widely used in many deep learning models to avoid over-fitting, can be interpreted to represent model uncertainty as a Bayesian approximation \cite{dropoutasuncertainty}. In some works, uncertainty is also used to address the issues of low-quality training data. In person re-identification, DistributionNet \cite{dropoutasuncertainty} adopts uncertainty to model the person images of noise-labels and outliers. In face recognition, DUL \cite{dul} and PFE \cite{Shi_2019_ICCV} apply data uncertainty to simultaneously learn the feature embedding and its uncertainty, where the uncertainty is learned through a learnable subnetwork to describe the quality of the image. In contrast to the aforementioned works, our proposed method is designed to model the feature statistics uncertainty under potential domain shifts and acts as a feature augmentation way for handling out-of-distribution generalization problems.

\subsection{Test-time Adaptation}
Test-time adaptation aims to adapt the trained model to the target domain without accessing the source data in the test time. Tent \cite{tent} adapts the model parameter of the batch normalization (BN) layers by test-time entropy minimization. AdaBN \cite{adabn} proposes to modulate the statistics in the BN layers with counterparts calculated by all images of the target domain. Alpha-BN \cite{alphabn} calibrates statistics in BN layers by mixing source and target domain statistics with a pairwise class correlation loss. SHOT \cite{shot} adopts the entropy minimization with pseudo labeling for test-time adaptation. AdaContrast \cite{adacontrast} leverages contrastive learning and online refinement to learn better representations. 
Our proposed inference-time adaptation strategy is different from these methods. Our method does not perform model updating or any additional training for inference-time adaptation, but adaptively calibrates the individual sample’s feature statistics on-the-fly at test time, which is efficient and also very suitable for domain generalization where the test dataset is unavailable.
For the DG problem, the previous methods generally require additional model tuning to perform adaptation, which can be inefficient and might even no longer suitable owing to the unavailability of test dataset.


\section{Methodology}
In this section, we propose our solution for tackling the domain generalization problem. First, we introduce the preliminaries in Section \ref{sec: preliminaries}. Then, we introduce the training-time uncertainty modeling method in Section \ref{sec: uncertainty modeling} and inference-time adaptation strategy in Section \ref{sec: inference-time adaptation}. Finally, Section \ref{sec: overall} summarizes the overall framework of our method.

\subsection{Preliminaries}
\label{sec: preliminaries}

Given $x \in \mathbb{R}^{B\times C \times H \times W}$ to be the encoded features in the intermediate (hidden) layers of the network, we denote the instance-level feature statistic $\mu(x) \in \mathbb{R}^{B \times C}$ and $\sigma(x) \in \mathbb{R}^{B \times C}$ as the channel-wise feature mean and standard deviation of each instance in a mini-batch, respectively, which can be formulated as:
\begin{align}
    &\mu(x)=\frac{1}{HW}\sum\limits_{h=1}^{H}\sum\limits_{w=1}^{W} x_{b,c,h,w}, \\
    &\sigma^2(x)=\frac{1}{HW}\sum\limits_{h=1}^{H}\sum\limits_{w=1}^{W} (x_{b,c,h,w} - \mu(x))^2.
\end{align}
As the abstraction of features, feature statistics can capture informative characteristics of the corresponding domain (such as color, texture, and contrast), according to previous works \cite{adain, li2021feature}. In out-of-distribution scenarios, the feature statistics often show inconsistency with the training domain due to different domain characteristics \cite{nips19trans,rbn}, which is ill-suited to deep learning modules like nonlinearity layers and normalization layer and degenerates the model's generalization ability \cite{nips20calibration}.

However, most of the deep learning methods only treat feature statistics as deterministic values measured from the features while lacking explicit consideration of the potential uncertain statistical discrepancy. Owing to the model's inherent vulnerability to such discrepancy, the generalization ability of the learned representations is limited.
Some recent methods \cite{Nuriel_2021_CVPR,mixstyle, efdm} utilize feature statistics to tackle the domain generalization problem. Despite the success, they typically adopt linear manipulation (\textit{i.e.}, exchange and interpolation) on pairwise samples to generate new feature statistics, which limits the diversity of synthetic changes. Specifically, the direction of their variants is determined by the chosen reference sample, and such internal operation restricts their variant intensity. Thus these methods are sub-optimal when handling the diverse and uncertain domain shifts in real world.


\begin{figure*}[t]
\begin{center}
\includegraphics[width=0.9\linewidth]{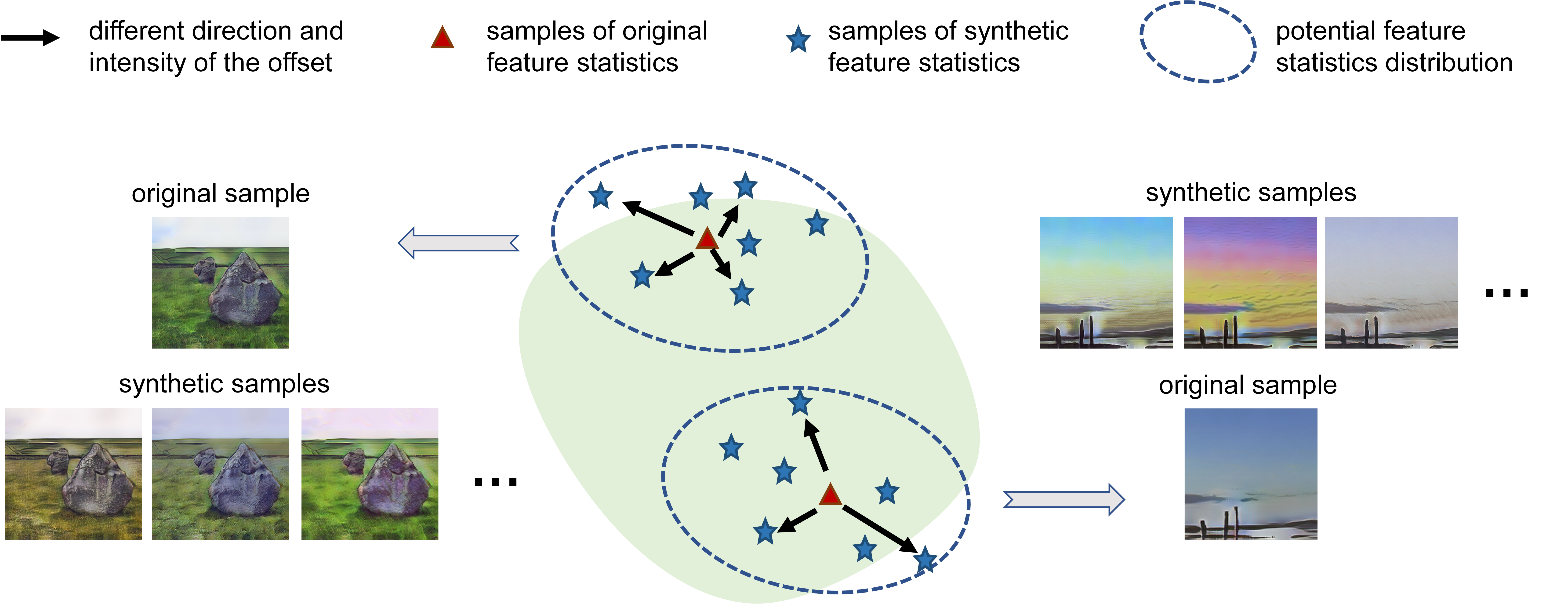}
\caption{
The visualization of reconstructed samples with synthesized feature statistics, using a pre-trained style transfer auto-encoder \cite{adain}. The illustration of the feature statistics shifts, which may vary in both intensity and direction (\textit{i.e.}, different offsets in the vector space of feature statistics). We also show images of ``new'' domains generated by manipulating feature statistic shifts with different directions and intensities. Note these images are for visualization only, rather than feeding into the network for training.
}
\label{fig:extra}
\end{center}
\vspace{-0.5cm}
\end{figure*}


\subsection{Modeling Domain Shifts with Uncertainty}
\label{sec: uncertainty modeling}

Since the underlying test domain is unforecastable and can be arbitrary, the uncertain feature statistic shifts may vary in both direction and intensity in vector space. Therefore, how to properly model the domain shifts becomes an important task for tackling the challenge of the domain generalization problem.


Considering the uncertainty and randomness of domain shifts, it is promising to employ the methods of “uncertainty” to treat the “uncertainty” of domain shifts. In this paper, we propose a novel method by \textit{modeling} \textbf{\emph{D}}\textit{omain}  \textbf{\emph{S}}\textit{hifts with}  \textbf{\emph{U}}\textit{ncertainty} (DSU). Instead of treating each feature statistic as a deterministic value measured from the learned feature, we hypothesize that the distribution of each feature statistic, after considering potential uncertainties, follows a multi-variate Gaussian distribution. This means each feature statistic has a probabilistic representation drawn from a certain distribution, \textit{i.e.}, the feature statistics mean and standard deviation follow $\mathcal{N}(\mu,\Sigma_{\mu}^2)$ and $\mathcal{N}(\sigma,\Sigma_{\sigma}^2)$, respectively.
Specifically, the corresponding Gaussian distribution's center is set as each feature's original statistics, while the Gaussian distribution's standard deviation describes the uncertainty scope for different potential shifts. Through randomly sampling diverse synthesized feature statistics with the probabilistic approach, the models can be trained to improve the robustness of the network against statistics shifts.

\subsubsection{Uncertainty Estimation}


Taking the uncertainty of domain shifts into consideration, the uncertainty estimation in our method aims to depict the uncertainty scope of each probabilistic feature statistic. However, the distribution of the testing domain is unknown, which makes it challenging to obtain an appropriate variant range during training.

As shown in some generative-based studies \cite{shen2021closedform, ISDA}, the variances between features contain implicit semantic meaning and can provide important indications. For example, the directions with a large intensity of variances can imply their potential for more valuable semantic changes.
Inspired by this, we propose a simple yet effective non-parametric
method for uncertainty estimation, utilizing the variance of the feature statistics in mini-batch to provide reasonable instructions about variant scope: 
\begin{equation}
\label{Sigmamu}
            \Sigma^2_{\mu}(x)=\frac{1}{B}\sum\limits_{i=1}^{B}(\mu(x_i)-\mathbb{E}_{b}[\mu(x)])^2,
\end{equation}
\begin{equation}
\label{Sigmasigma}
    \Sigma^2_{\sigma}(x)=\frac{1}{B}\sum\limits_{i=1}^{B}(\sigma(x_i)-\mathbb{E}_{b}[\sigma(x)])^2,
\end{equation}
where $\Sigma_{\mu}(x)\in \mathbb{R}^C$ and $\Sigma_{\sigma}(x)\in \mathbb{R}^C$ represent the uncertainty estimation of the feature mean $\mu$ and feature standard deviation $\sigma$, respectively. 
The magnitudes of uncertainty estimation can reveal the possibility that the corresponding channel may change potentially. 
For example, if the uncertainty estimation in some channel is large, which might indicate this channel is rich in semantics and has a bigger possibility to change across different domains. 
Although the underlying distribution of the domain shifts is unpredictable, the uncertainty estimation based on the training observations can provide an appropriate and meaningful variation range for each feature channel, which does not harm model training but can simulate diverse and reasonable potential shifts.

\subsubsection{Probabilistic distribution of feature statistics}

Once the uncertainty estimation of each feature channel is obtained, the Gaussian distribution for probabilistic feature statistics can be established. 
To further explore the uncertainty, we adopt random sampling to exploit the randomness in the probabilistic representations. The new feature statistics, mean $\beta(x) \sim \mathcal{N}(\mu,\Sigma_{\mu}^2)$ and standard deviation $\gamma(x) \sim \mathcal{N}(\sigma,\Sigma_{\sigma}^2)$, can be randomly drawn from the corresponding distributions as:


\begin{equation}
\label{dsu_training1}
    \beta(x)=\mu(x)+\epsilon_{\mu}{\Sigma}_{\mu}(x),
    ~~~~~~\epsilon_{\mu}\sim\mathcal{N}(\mathbf{0},\mathbf{I}),
\end{equation}
\begin{equation}
\label{dsu_training2}
    \gamma(x)=\sigma(x)+ \epsilon_{\sigma}{\Sigma}_{\sigma}(x), ~~~~~~\epsilon_{\sigma}\sim\mathcal{N}(\mathbf{0},\mathbf{I}).
\end{equation}

Since the proposed method is a stochastic algorithm, we use the re-parameterization technique \cite{vae} to make the sampling operation differentiable. Then we can get Equation (\ref{dsu_training1},\ref{dsu_training2}), where $\epsilon_{\mu}$ and $\epsilon_{\sigma}$ both follow the standard Gaussian distribution. To trade off the strength of uncertainty modeling during training, we set a hyper-parameter $p$ that denotes the probability to apply it. By exploiting the given Gaussian distribution, random sampling can generate various new feature statistics information with different combinations of directions and intensities. 
As we model the feature statistics as uncertain distributions during training, the model is trained to be more robust to statistic shifts. It can be observed that the distribution gaps of feature statistics are narrowed by modeling uncertainty during training (shown in Fig. \ref{fig:shift}), achieving an implicit effect to be robust against domain shifts.

\subsection{Inference-time Adaptation}
\label{sec: inference-time adaptation}

Modeling uncertainty during training can provide a reasonable variant scope based on the observations and generate kinds of shifts with different strengths and directions to enhance the trained model's robustness to domain shifts. During inference, the trained model becomes static and will encounter different underlying test distributions, which are unpredictable and might have inevitable gaps to training observations. To enhance the generalization ability of the static trained model to various test data, we can establish an inference-time adaptation strategy based on the instance-wise relative shift to the training domain of instance-wise feature statistics. 

\begin{figure}[t]
\begin{center}
\setlength{\abovecaptionskip}{0.cm}
\setlength{\belowcaptionskip}{-0.cm}
\includegraphics[width=0.9 \linewidth]{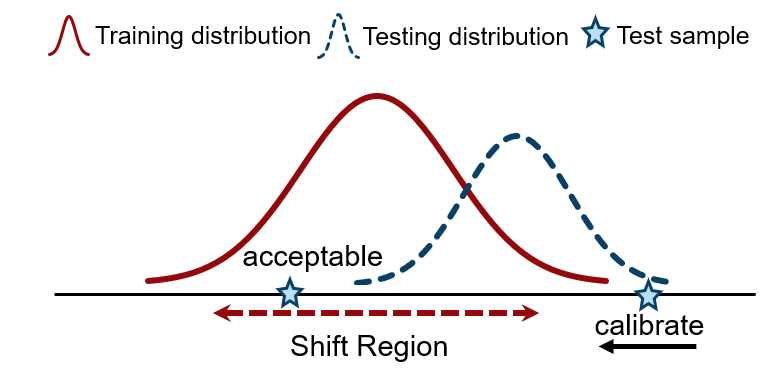}
\end{center}
\caption{Illustration of the inference-time adaptation strategy on instance-wise feature statistics. If the feature statistics of a test sample fall beyond the shift region of training domains, we can provide an adaptive calibrated intervention to help the model resist the perturbation. }
\label{fig:calibrate}
\vspace{-0.5cm}
\end{figure}

Therefore, we further propose an inference-time instance adaptation strategy to adaptively calibrate the feature statistic of test data, \textit{i.e.,} if the shift is beyond a given shift region related to the training domain and we can explicitly adapt the feature statistic towards the training distribution. Specifically, we define the shift region of feature statistics (mean and standard deviation ) as $\{\mathcal{S}_{\mu}$, $\mathcal{S}_{\sigma}\}$ to represent a certain variant scope of training distribution:

\begin{equation}
\label{accepetable1}
\mathcal{S}_{\mu}=[\overline{\mu}-n\overline{\Sigma}_{\mu}, \overline{\mu}+n\overline{\Sigma}_{\mu}],
\end{equation}
\begin{equation}
\label{accepetable2}
\mathcal{S}_{\sigma}=[\overline{\sigma}-n\overline{\Sigma}_{\sigma}, \overline{\sigma}+n\overline{\Sigma}_{\sigma}],
\end{equation}
where $\{\overline{\mu},\overline{\sigma}\}$ and $\{\overline{\Sigma}_{\mu},\overline{\Sigma}_{\sigma}\}$ denote the center and standard deviation of training domain feature statistics (mean and standard deviation) distribution. 
As illustrated in Fig. \ref{fig:calibrate}, if the feature statistic of a test sample exceeds the shift region and is away from the training distribution, it indicates the domain shift might be hard to estimate from training observations. Therefore, we can perform an explicit intervention and calibrate its statistics toward the training distribution with an offset. Conversely, if the feature statistic falls into the shift region $\{\mathcal{S}_{\mu}$, $\mathcal{S}_{\sigma}\}$, then we consider the model can be well trained to tolerate this perturbation and do not need intervention.
As a result, we raise an inference-time adaptation strategy as follows:
\begin{equation}
\label{dsu_testing1}
\beta(x)=\mu(x)+\omega\cdot\text{sign}(\overline{\mu}-\mu(x))\cdot\Phi(|\overline{\mu}-\mu(x)|-n\overline{\Sigma}_{\mu}),
\end{equation}
\begin{equation}
\label{dsu_testing2}    
\gamma(x)=\sigma(x)+\omega\cdot\text{sign}(\overline{\sigma}-\sigma(x))\cdot\Phi(|\overline{\sigma}-\sigma(x)|-n\overline{\Sigma}_{\sigma}),
\end{equation}
where $\Phi(x)\text{=max}(x,0)$ is the ReLU function to represent the adaptive offset, $\text{sign}(x)\in\{\text{-1,0,1}\}$ is the sign function to denote the relative position of the test sample, and $\omega$ controls the calibration strength towards the shift region. Through the Equation (\ref{dsu_testing1}, \ref{dsu_testing2}), we can on-the-fly achieve the adaptative calibration on the feature statistics of the test sample.

The raised inference-time adaptation strategy further bridge the gap between the static model and various unpredictable shifts. Therefore the trained model can adaptively calibrate the unpredictable test sample and be more robust to diverse domain shifts. We note that the inference-time strategy only calibrates the statistics using the individual test sample and the training domain distribution information, which does not need to adapt the model parameter or additional training, and thus is suitable and efficient for domain generalization.

\subsection{Overall Framework}
\label{sec: overall}
The implementation of our method is to replace the new feature statistics from the proposed method to achieve the transformation as in AdaIN \cite{adain}. Overall, the final form of the proposed method can be formulated as: 
\begin{equation}
    \hat{x}=\gamma_{\text{dsu}}(x) \frac{x-\mu(x)}{\sigma(x)} + \beta_{\text{dsu}}(x).
\end{equation}
During training, the new feature statistics $\beta_{\text{dsu}}(x)$ and $\gamma_{\text{dsu}}(x)$ are randomly sampled from the corresponding Gaussian distribution by the DSU module, as in Equation (\ref{dsu_training1},\ref{dsu_training2}) to model uncertainty, while during inference the feature statistics $\beta_{\text{dsu}}(x)$ and $\gamma_{\text{dsu}}(x)$ are calibrated by the adaptation strategy, as in Equation (\ref{dsu_testing1},\ref{dsu_testing2}) to perform calibration, together forming a system-level solution for improving domain generalization ability. We name this solution as DSU++ and the algorithm is described in Algorithm \ref{alg:dsu}.

The above operations can be flexibly integrated at arbitrary stages of the network as a plug-and-play module.
Benefiting from the proposed method, the model trained with uncertain feature statistics will gain better robustness against potential statistics shifts, and thus acquires a better generalization ability. Besides, the inference-time adaptation module can also play a role in adaptively calibrating the unforecastable domain shifts. The whole procedure is effective and efficient, requiring no additional trainable parameters during training and not needing to adapt the model parameters during inference.

\begin{algorithm}[htp]
\small 
\caption{The algorithm of our method DSU++.}
\label{alg:dsu}
\KwIn{
Intermediate feature $x\in \mathbb{R}^{B\times C\times H\times W}$, probability $p$ to forward this module\;
}
\KwOut{
Intermediate feature $\widehat{x} \in \mathbb{R}^{B\times C\times H\times W}$ after considering potential statistics shifts\;
}


Sample $p_0 \sim U(0,1)$\;
\uIf{$p_0<p$ \bf and Training}{
    Compute the channel-wise mean and standard deviation of each instance in a mini-batch\;
    
    $\mu(x)=\frac{1}{HW}\sum\nolimits_{h=1}^{H}\sum\nolimits_{w=1}^{W} x_{b,c,h,w}$,

     $\sigma^2(x)=\frac{1}{HW}\sum\nolimits_{h=1}^{H}\sum\nolimits_{w=1}^{W} (x_{b,c,h,w} - \mu(x))^2$.
    
    Uncertainty estimation on feature statistics\;
    
    $\Sigma_{\mu}^2(x)=\frac{1}{B}\sum\nolimits_{i=1}^{B}(\mu(x_i)-\mathbb{E}_{b}[\mu(x)])^2$,

    $\Sigma_{\sigma}^2(x)=\frac{1}{B}\sum\nolimits_{i=1}^{B}(\sigma(x_i)-\mathbb{E}_{b}[\sigma(x)])^2$.

    Compute the synthetic feature statistics randomly sampling from the given Guassian distributions\;

    $\beta(x)=\mu(x)+\epsilon_{\mu}{\Sigma}_{\mu}(x),
     ~~~~\epsilon_{\mu}\sim\mathcal{N}(\mathbf{0},\mathbf{I})$,
     
    $\gamma(x)=\sigma(x)+ \epsilon_{\sigma}{\Sigma}_{\sigma}(x), ~~~~\epsilon_{\sigma}\sim\mathcal{N}(\mathbf{0},\mathbf{I})$.
     
    Obtain the feature after considering potential statistics shifts\;
    
    $\widehat{x}=\gamma(x)\times \frac{x-\mu(x)}{\sigma(x)} + \beta(x)$.
     
     return the feature $\widehat{x}$ with uncertain feature statistics.
}
\uElseIf{\bf Testing}{
    \label{tta}
    $\beta(x)=\mu(x)+\omega\cdot\text{sign}(\overline{\mu}-\mu(x))\cdot\Phi(|\overline{\mu}-\mu(x)|-n\overline{\Sigma}_{\mu}),$\\
    $\gamma(x)=\sigma(x)+\omega\cdot\text{sign}(\overline{\sigma}-\sigma(x))\cdot\Phi(|\overline{\sigma}-\sigma(x)|-n\overline{\Sigma}_{\sigma}).$\\
    Obtain the feature after feature statistics calibration;
    $\widehat{x}=\gamma(x)\times\frac{x-\mu(x)}{\sigma(x)} + \beta(x)$.
    
    return the feature $\widehat{x}$ after the inference-time adaptation module.
}
\Else{
    adopt the original feature $x$ and skip this module.  
}
\end{algorithm}

\section{Theoretical Analysis}
\label{theory}
\newtheorem{theorem}{Theorem}[section]
We have presented the methodology in the above sections. In this section, we conduct a detailed theoretical analysis about the efficacy of our method, which is organized as follows.
We first present the required definitions in section \ref{sec:definitions}. Under the definitions, we provide an intensive analysis to show how our method reduces the risk brought by the domain gap, from the perspective of \textit{generalization error bound} in section \ref{sec: generalization error bound} and \textit{implicit regularization} in section \ref{sec:implicit regularization effect}, respectively. The detailed proof is in the appendix \ref{sec: appendix}.

\newtheorem{assumption}{Assumption}[section]
\newtheorem{definition}{Definition}[section]
\newtheorem{lemma}{Lemma}[section]
\subsection{Definitions} 
\label{sec:definitions}
In previous domain generalization problem setting \cite{cha2021swad},\cite{gao2022loss}, it is assumed that the source domain $\mathcal{S}$$:=$$\{\mathcal{S}_{i}\}_{i=1}^{S}$ and  $\mathcal{T}$$:=$$\{\mathcal{T}_i\}_{i=1}^{T}$ are constituted of
several available training domains and several unseen testing domains, where both $\mathcal{S}_i$ and $\mathcal{T}_i$ are training and testing distributions over input space $\mathcal{X}$, respectively. The goal of domain generalization problem is to improve performance on all possible domains, rather than simply a few target domains \cite{Robey2021ModelBasedDG}. In order to better formalize the potential target domains and meet the goal, we provide a probabilistic modeling approach, where an important modification is that we assume the data comes from an infinite number of domains determined by the environment variable $e$. Generally, the environment variable $e$ varies from changes in domain characteristics, \textit{i.e.,} texture, saturation, illumination, and background, etc. We define the target domain as $\mathcal{T}$$:=$$\Pi_e\cdot\mathcal{D}_e$, where $\Pi_e$ is a continuous random variable that represents the probability of the occurrence of environment $e$ and $\mathcal{D}_e$ is the domain of the corresponding environment. Naturally, the set of source domains is determined by several environments $e_i$ and we can describe the source training domain as a mixture distribution:
$\mathcal{S}$$:=$$\sum_{i=1}^{S}\Pi_{i}\cdot\mathcal{D}_{e_i}$, where $S$ is the total number of source domains and $\Pi_i$ is the mixing probability of domain $\mathcal{D}_{e_i}$. 

Similar to previous theoretical works, we consider the binary classification task for analysis. We assume that there exists a global labeling function $f:\mathcal{X}\rightarrow [0,1]$ corresponds to the probability that the label of $x$ is 1. Then, the optimized objective from the image level is to use a target function $\widetilde{f}:\mathcal{X}\rightarrow [0,1]$ to fit the ground-truth function $f$:

\begin{equation}
    \text{arg }\mathop{\text{min}}_{\theta}\mathbb{E}_e(\Pi_e\cdot \mathbb{E}_{x\sim\mathcal{D}_{e}}l(\widetilde{f}(x,\theta),f(x))),
\end{equation}
where $\widetilde{f}(·; \theta)$ is the target function parameterized by $\theta$ and $l:[0,1]\times[0,1]\rightarrow R^{+}$ is the loss function. 

To conduct the analysis on the fitting ability of intermediate layers in network, it requires us to extend the objective from
the image level to feature level. 
Consider the top-k layers in the network: $h_k:\mathcal{X}\rightarrow \mathcal{Z}_k$ that maps the input from the image space to the corresponding feature space and we can naturally obtain a probability function in the feature space as follows.
\begin{definition}
\label{feature target funtion}
\textbf{(Target function in feature space \cite{NIPS2006_b1b0432c})} A representation function $h_k$ induces a distribution $\widetilde{\mathcal{D}}$ over $\mathcal{Z}_k$ and a target function $f_
k:\mathcal{Z}_k\rightarrow [0,1]$ as follows:
\begin{equation}
\begin{split}
    &Pr_{\widetilde{\mathcal{D}}}(A)=Pr_{\mathcal{D}}(h_k^{-1}(A)), \\
    f&_{k}(z)=E_{\mathcal{D}}[f(x)|h_k(x)=z].
\end{split}
\end{equation}
for any $A\subset Z$ such that $h_k^{-1}(A)$ is measurable in $\mathcal{D}$. $Pr_X(x)$ denotes the probability of $x$ in distribution $X$. $\widetilde{\mathcal{D}}$ denote the feature distributions obtained by mapping the image distribution of $\mathcal{D}$ into feature space. Similarly, The source and target feature distributions $\widetilde{\mathcal{S}}, \widetilde{\mathcal{T}}$ can be obtain from mapping $\mathcal{S}, \mathcal{T}$ into feature space. 
\end{definition}

Through the Definition \ref{feature target funtion}, the training objective of domain generalization problem can be extended from the image level to the feature level, which becomes to find a predictor in feature space $\widetilde{f}_k:\mathcal{Z}_k\rightarrow [0,1]$ that agrees with $f_k$ on target domain: 
\begin{equation}
\begin{split}
    &\quad\,\text{arg }\mathop{\text{min}}_{\theta}\mathbb{E}_e(\Pi_e\cdot \mathbb{E}_{x\sim\mathcal{D}_{e}}l(\widetilde{f}_k(h_k(x),\theta),f(x))) \\
    &=\text{arg }\mathop{\text{min}}_{\theta}\mathbb{E}_e(\Pi_e\cdot \mathbb{E}_{z\sim\widetilde{\mathcal{D}_{e}}}l(\widetilde{f}_k(z,\theta),f_k(z))).
\end{split}    
\end{equation}
\par  For convenience, we denote the error that the predictor $\widetilde{f}_k$ disagrees with $f$ on training distribution $\mathcal{S}$ as $\epsilon_{\mathcal{S}}(\widetilde{f}_k, f_k):=\mathbb{E}_{z\sim \widetilde{\mathcal{S}}}|\widetilde{f}_k(z)-f_k(z)|$. Similarly, $\epsilon_{\mathcal{T}}(\widetilde{f}_k, f_k)$  denotes the expected error of $\widetilde{f}_k$ on testing domain $\mathcal{T}$. The sum $\epsilon_{\mathcal{S}}(\widetilde{f}_k, f_k)+\epsilon_{\mathcal{T}}(\widetilde{f}_k, f_k)$ represents the disagreement of $\widetilde{f}_k$ and $f_k$ on source domain and target domain. Then, we can use the upper bound $\lambda$ of this sum to evaluate the approximation ability of the function class $\mathcal{H}$ as follows.
\begin{definition}
\label{close}
\textbf{($\lambda$-close \cite{NIPS2006_b1b0432c
})} For $\lambda>0$, a function $f:\mathcal{Z}_k\rightarrow [0,1]$ is $\lambda$-close to  a function class $\mathcal{H}$ if 
\begin{equation}
    \mathop{inf}\limits_{h\in H}[\epsilon_{\mathcal{S}}(h, f)+\epsilon_{\mathcal{T}}(h, f)]\leq \lambda
\end{equation}
\end{definition}
As is believed by many works \cite{NIPS2006_b1b0432c}, there is a function in $\mathcal{H}$ similar to the objective function $f_k$, which embodies $f_k$ is $\lambda$-close to the function class $\mathcal{H}$, and the $\lambda$ should be a small quantity due to powerful fitting ability of deep neural networks. This assumption is naturally acquiesced in this paper and will be used in the following analysis.
\subsection{Generalization Error Bound}
\label{sec: generalization error bound}
We conduct a theoretical analysis under the above definitions and obtain the feature-level \textit{generalization error bound} on target domain. From that, we theoretically provide a boarder view to understand how our method reduces the generalization risk brought by domain shift in both training and inference stages.
\begin{theorem}
\label{Generalization bound}
\textbf{(Generalization gap bound)} Consider the representation function $h_k:\mathcal{X}\rightarrow \mathcal{Z}_k$ and a Lipschitz continuous function class $\mathcal{H}$ with Lipschitz constant $L$. If $X = \{x_1$, . . . , $x_n\}$ is a set of samples drawn from $\mathcal{S}$-i.i.d labeled according to $f$. Then, for any $h \in \mathcal{H}$, the following bound holds with probability at least $1-\delta$$:$  
\begin{equation}
\label{inequality}
    \begin{split}
        \epsilon_{\mathcal{T}}(h)&\leq \hat{\epsilon}_{\mathcal{S}}(h)+2L\sum\limits_{i=1}^{S}\Pi_{i}\cdot \int_{e} \Pi_e\cdot \mathcal{W}(\widetilde{\mathcal{D}}_{e_i},\widetilde{\mathcal{D}}_e)   \,\mathsf{d}e \\
        &+\lambda+2\hat{\mathcal{R}}_X(\mathcal{H'})+3\sqrt{\frac{log(\frac{2}{\delta})}{2n}}.
    \end{split}
\end{equation}
where $\hat{\epsilon}_{\mathcal{S}}(h)=\frac{1}{n}\sum\limits_{i=1}^n |h(h_k(x_i))-f(x_i)|$ is the empirical risk, $\mathcal{W}(P_1, P_2)=\mathop{\text{inf}}\limits_{\gamma\sim \Pi(P_1, P_2)}\mathbb{E}_{(x,y)\sim\gamma}||x-y||_2$ is the Wasserstein distance to measure the distance between two domains, $\mathcal{H'}:=\{|h\circ h_k - f|\ |h\in \mathcal{H}\}$ and the $\hat{\mathcal{R}}_X$ is empirical Rademacher Complexity which we give a detailed introduction in appendix \ref{rademacher}. Proof is in the Appendix \ref{Appendix: generalization error bound}.
\end{theorem}
The above introduced theorem shows that the generalization error on the target
domain $\epsilon_{\mathcal{T}}(h)$ is bounded by five terms: The first term is the empirical training loss $\epsilon_{\widetilde{\mathcal{S}}}(h)$ optimized as the objective function of ERM principle \cite{erm}, which can be optimized to be relative small. The second term is Wasserstein distance between source domain and target domain, which is well-known as the Earth Mover's Distance for the optimal cost needed to transport the mass from one distribution to another. The third term $\lambda$ is a small quantity to represent the disagreement as defined in Definition \ref{close}. The fourth term is empirical Rademacher Complexity of function class $\mathcal{H'}$. The fifth term is a constant associated to the degree of confidence and the number of training samples. 

Because other terms are relatively small or constant, Theorem \ref{Generalization bound} provides guarantee that the generalization error on target domain can be controlled by the Wasserstein distance between source domain and target domain. Through decreasing the term of Wasserstein distance, the generalization risk can be reduced from lowering the upper bound in the right hand of Equation (\ref{inequality}).
Empirically, if two domains obey Gaussian distribution ${\displaystyle X\sim{\mathcal {N}}(\mu_{1},\Sigma_{1})}$ and ${\displaystyle  Y\sim{\mathcal {N}}(\mu_{2},\Sigma_{2})}$. Then, the Wasserstein distance between two distributions is shown to be correlated to statistic discrepancy as follows:
\begin{equation}
    \begin{split}
        \displaystyle \displaystyle \mathcal{W}(X, Y)^{2}&=\|\mu_{1}-\mu_{2}\|_{2}^{2} \\
        &+\mathrm {Tr}{\bigl (}\Sigma_{1}+\Sigma_{2}-2{\bigl (}\Sigma_{2}^{\frac{1}{2}}\Sigma_{1}\Sigma_{2}^{\frac{1}{2}}{\bigr )}^{\frac{1}{2}}{\bigr )},
    \end{split}
\end{equation}
which implies that reducing the statistical discrepancy between domains can achieve the goal of reducing the Wasserstein distance between domains and thus the upper bound of the generalization error on the target domain in Equation (\ref{inequality}) can be theoretically decreased.

Recalling that we model the uncertainty of the statistic distribution during training to model the potential various domain shifts, which enhances the model's robustness to statistic shifts and implicitly reduces the distance between the source domain and the unseen domain. In the inference stage, the method performs calibration for the data whose statistics deviate from the source domain to explicitly reduce the distance. Therefore, we can obtain theoretical guarantees for the proposed methods to act on reducing the generalization risk, supported by the introduced framework of generalization error bound in feature space.
\subsection{Implicit Regularization of DSU}
\label{sec:implicit regularization effect}
The above generalization error bound provides a general perspective to understand the efficacy of our method. Moreover, regarding the proposed method DSU as a stochastic algorithm, it is important but challenging to uncover how it influences the training process. To this end, we conduct a specialized analysis to reveal its working mechanism and theoretically prove that DSU can bring an implicit regularization effect for the training dynamic process under some assumptions. To simplify the analysis, we conduct the study on the widely used regression problem. 

\begin{theorem}
\label{Implicit regularization}

Consider the classic regression problem with a family $\mathcal{F}$ of MLP functions $f_{w, b}(x) = w\cdot x+b$, where $x\in R^{C\times N}$ is the output of k-th layer in the network, $w\in R^{C\times N}$ and $b\in R$ are the affine coefficients, $w\cdot x$ represents the dot product operation. The loss function is $R(f) = \frac{1}{n}\sum_{i=1}^{n}(f(x_i)-y_i)^2 $ and we have the expectation of loss function after the probabilistic uncertainty modeling in k-th layer as:
\begin{equation}
\begin{split}
    \mathbb{E}[R(f)] &= \frac{1}{n}\sum_{i=1}^{n}\mathbb{E}_{\epsilon_{\sigma}, \epsilon_{\mu}}[(w\cdot ((\sigma(x_i)+ \epsilon_{\sigma}{\Sigma}_{\sigma}(x))\frac{x_i-\mu(x_i)}{\sigma(x_i)} \\
    &+\mu(x_i)+\epsilon_{\mu}{\Sigma}_{\mu}(x))+b-y_i)^2]. \\
\end{split}
\end{equation}
We prove that it equals to the following equation, detailed proof can be found in the Appendix \ref{Appendix: Implicit regularization}:
\begin{equation}
\label{regression_exp}
\begin{split}
    \mathbb{E}[R(f)] &= \frac{1}{n}\sum_{i=1}^{n}(f(x_i)-y_i)^2+ 
    \sum_{i=1}^{C} (\Sigma_{\mu})_i^2\cdot ||w_i||_2^{2}  \\ &+\frac{1}{n}
    \sum_{i=1}^{C}(\Sigma_{\sigma})_i^2\cdot\sum_{j=1}^{n}||w_i||_2^{2}\cdot \langle w_i,(\frac{x_j- \mu (x_j)}{\sigma (x_j)})_i\rangle^2,
\end{split}
\end{equation}
where $(\cdot)_i$ represents the i-th channel, the expected value of a random variable X is denoted by $\mathbb{E}[X]$, $\langle x, y\rangle$ represents the cosine distance. The notifications $\Sigma_{\sigma}$, $\Sigma_{\mu}$, $\epsilon_{\sigma}$, and $\epsilon_{\mu}$ are same as previously defined in Equation (\ref{Sigmamu}, \ref{Sigmasigma}, \ref{dsu_training1}, \ref{dsu_training2}).
\end{theorem}
Equation (\ref{regression_exp}) demonstrates that the probabilistic approach is expected to bring an implicit regularization effect to the training process. Compared to the vanilla version, DSU introduces a penalty on the channel-wise variances of source domain feature statistics, $(\Sigma_{\sigma})_i$ and $(\Sigma_{\mu})_i$, according to the coefficients $||w_i||_2$ and $\langle w_i,(\frac{x_j- \mu (x_j)}{\sigma (x_j)})_i\rangle^2$. Specifically, the coefficients are proportional to the weight of the corresponding channel, meaning that channels with large importance can get intensive constraints on variances. 

It is shown that, as we produce diverse domain perturbations during training and enlarge the training scope with uncertainty, the training process in turn obtains a penalty to reduce the variances of feature statistics. Benefiting from the dynamic enlarging-and-reduction training process, the model can be trained to gain robustness to the domain perturbation, which implicitly
aligns the feature statistics of multiple domains and helps the model encode better domain-invariant features across different domains. Also, it has the benefit of bringing different regularization effects in each step from the stochastic way. As a result, the training domains can become more compact and the domain distance between source and target domain is implicitly reduced, the empirical analysis in Section \ref{analysis} verifies our conclusion as well. 



\section{Experiments}

In order to verify the effectiveness of the proposed method in improving the generalization ability of networks, we conduct experiments on a wide range of tasks, including image classification, semantic segmentation, instance retrieval, and pose estimation, where the training and testing sets have different cases of distribution shifts, such as style shift, synthetic-to-real gap, scenes change, and different animal kinds. In the following sections, we denote DSU++ as the experiment results equipped with both DSU and the inference-time adaptation strategy. 

\subsection{Generalization on Multi-domain Classification}
\textbf{Setup and Implementation Details:} We evaluate the proposed method on PACS \cite{pacs_2017_ICCV}, a widely-used benchmark for domain generalization with four different styles: Art Painting, Cartoon, Photo, and Sketch. The implementation follows a \textit{leave-one-domain-out} protocol and ResNet18 \cite{resnet} is used as the backbone. The random shuffle version of MixStyle is adopted for fair comparisons, which does not use domain labels.

\begin{table}[h]
\small
\centering
\caption{Experiment results of PACS multi-domain classification task.}
\label{table:pacs}
\resizebox{\linewidth}{!}{
\begin{tabular}{l|c|cccc|c}
\toprule[1pt]
\multicolumn{1}{l|}{Method} & Reference & Art & Cartoon & Photo & Sketch & Ave. (\%)  \\ \hline
Baseline&- &74.3  &76.7  & 96.2  &68.7  & 79.0\\
Mixup \cite{zhang2018mixup}&ICLR 2018 & 76.8 & 74.9 & 95.8 & 66.6 & 78.5 \\
Manifold \cite{manifoldmixup}&ICML 2019 & 75.6 & 70.1 & 93.5  & 65.4 & 76.2\\
CutMix \cite{cutmix}&ICCV 2019 & 74.6 & 71.8 & 95.6 & 65.3 & 76.8 \\
RSC \cite{huangRSC2020} &ECCV 2020 & 78.9 &  76.9 & 94.1 & 76.8 & 81.7 \\
L2A-OT \cite{L2O} & ECCV 2020 & 83.3 & 78.2 & 96.2 & 73.6 & 82.8 \\
SagNet \cite{nam2021reducing}&CVPR 2021 &83.6 & 77.7 & 95.5 & 76.3 & 83.3 \\ 
pAdaIN \cite{Nuriel_2021_CVPR}&CVPR 2021 & 81.7 & 76.6 & 96.3 & 75.1 & 82.5 \\
MixStyle \cite{mixstyle}&ICLR 2021 & 82.3 & 79.0 & 96.3 & 73.8 & 82.8 \\
SFA \cite{li2021simple} & ICCV 2021 & 81.2 & 77.8 & 93.9 & 73.7 & 81.7 \\
EFDM \cite{efdm}&CVPR 2022 & 83.9 & 79.4 & \textbf{96.8} & 75.0 & 83.9 \\
\hline
DSU\cite{dsu}& ICLR 2022 & 83.6 & 79.6 & 95.8 & 77.6 & 84.1 \\
DSU++& this paper & \textbf{84.5} & \textbf{80.6} & 96.4 & \textbf{79.2} & \textbf{85.2} \\

\bottomrule[1pt]
\end{tabular}
}
\end{table}

\textbf{Experiment Results:} 
The experiments results, shown in Table \ref{table:pacs}, demonstrate our significant improvement over the baseline method, which shows our superiority to the conventional deterministic approach. 
Furthermore, our method also outperforms the competing methods, especially in Sketch, which has a divergent style to other domains and is more difficult. Compared to the internal operations on feature statistics, our method that provides diverse uncertain shifts on feature statistics is effective to improve network generalization ability against different domain shifts. With equipping the inference-time adaptation strategy, the performances of DSU++ can be further improved from 84.1 to 85.2, showing that training-time uncertainty modeling and inference-time adaptation can work together to help the model better deal with the domain shifts.
Photo has similar domain characteristics as ImageNet dataset and the slight drop might be due to the ImageNet pre-training (also discussed in \cite{randcov}). DSU augments the features and enlarges the diversity of the training domains. In contrast, the baseline method preserves more pre-trained knowledge from ImageNet thus tends to overfit the Photo style dataset benefiting from pre-training. 

\subsection{Generalization on Semantic Segmentation}
\textbf{Setup and Implementation Details:} Semantic segmentation, 
as a fundamental application for automatic driving, encounters severe performance declines due to scenario differences \cite{Haoran_2020_FADA}. GTA5 \cite{gtav} is a synthetic dataset generated from Grand Theft Auto 5 game engine and Synthia is a synthetic urban scene dataset. Cityscapes \cite{Cityscapes} is a real-world dataset collected from different cities in primarily Germany. To evaluate the cross-scenario generalization ability of segmentation models, we adopt synthetic GTA5 and Synthia for training while using real CityScapes for testing.
The experiments are conducted on FADA released codes \cite{Haoran_2020_FADA}, using DeepLab-v2 \cite{deeplab} segmentation network with ResNet101 backbone. Mean Intersection over Union (mIOU) of all object categories are used for evaluation.

\begin{table}[h]
\centering
\normalsize
\setlength{\tabcolsep}{4.5mm}
\caption{Experiment results of semantic segmentation from synthetic GTA5 and Synthia to real Cityscapes.}
\label{table:segmentation}
\resizebox{1.0\linewidth}{!}{
\begin{tabular}{l|c|c|c}

\toprule[1pt]
\multirow{2}{*}{Method} &\multirow{2}{*}{Reference}
& \multicolumn{1}{c|}{GTA5} 
& \multicolumn{1}{c}{Synthia}\\
 & & mIOU (\%)         & mIOU (\%) \\ \hline

Baseline & - & 37.0 & 32.4 \\
pAdaIN \cite{Nuriel_2021_CVPR}& CVPR 2021 & 38.3 & 32.3 \\
Mixstyle \cite{mixstyle}& ICLR 2021 & 40.3 &  33.2\\
EFDM \cite{efdm}& CVPR 2022 &  40.8&  32.6\\
\hline 
DSU \cite{dsu}& ICLR 2022& 43.1 & 33.9\\
DSU++ & this paper& \textbf{43.6} & \textbf{34.5} \\
\bottomrule[1pt]
\end{tabular}
}

\end{table}




\textbf{Experiment Results:} Table \ref{table:segmentation} shows the experiment results compared to related methods. As for a pixel-level classification task, improper changes in feature statistics might constrain the performances. The variants generated from our method are centered on the original feature statistics with different perturbations. These changes in feature statistics are mild for preserving the detailed information in these dense tasks. Meanwhile, DSU can take full use of the diverse driving scenes to generate diverse variants and the inference adaptation can further help the model deal with various unforecastable shifts, thus showing a significant improvement on mIOU by 6.6\%. We also provide the visualization result in Figure \ref{fig:seg1} to show its effectiveness. 

\begin{figure}[h]
\begin{center}
\setlength{\abovecaptionskip}{0.cm}
\setlength{\belowcaptionskip}{-0.cm}
\includegraphics[width=1.0\linewidth]{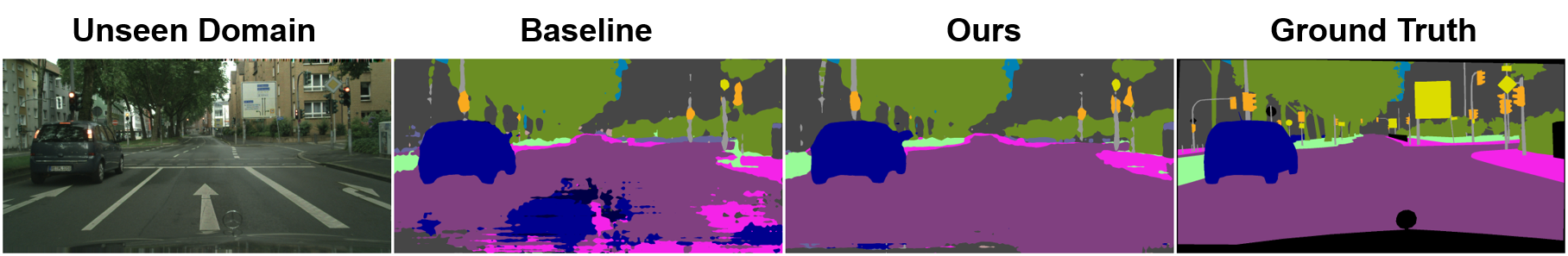}
\caption{The visualization on unseen domain Cityscapes with the model trained on synthetic GTA5. }
\vspace{-3pt}
\label{fig:seg1}
\end{center}

\end{figure}


\subsection{Generalization on Instance Retrieval}
\textbf{Setup and Implementation Details:} In this section, person re-identification (ReID), which aims at matching the same person across disjoint camera views, is used to verify the effectiveness of our method on the instance retrieval task. 
Experiments are conducted on the widely used DukeMTMC \cite{dukemtmc} and Market1501 \cite{market1501} datasets. The implementation is based on MMT \cite{mmt} released codes and ResNet50 is adopted as the backbone. Meanwhile, mean Average Precision (mAP) and Rank-1 (R1) precision are used as the evaluation criterions.

\begin{table}[h]
\setlength{\abovecaptionskip}{0.cm}
\setlength{\belowcaptionskip}{-0.cm}
\caption{Experiment results of instance retrieval on ReID dataset DukeMTMC and Market1501. A $\rightarrow$ B denotes models are trained on A while evaluated on B. For fair comparisons, we reproduce the experiments under the same framework. }
\label{table:reid}
\begin{center}
\resizebox{0.95\linewidth}{!}{
\begin{tabular}{l|c|cc|cc}
\toprule[1pt]
\multirow{2}{*}{Method} &\multirow{2}{*}{Reference}
& \multicolumn{2}{c|}{Market $\rightarrow$ Duke} 
& \multicolumn{2}{c}{Duke $\rightarrow$ Market}\\

  & & mAP (\%)           & R1 (\%)         & mAP (\%)            & R1 (\%) \\ \hline 
Baseline & - &25.8 & 42.3 & 26.7 & 54.7\\
pAdaIN \cite{Nuriel_2021_CVPR} & CVPR 2021 & 28.0 & 46.1 & 27.9 & 56.1 \\
MixStyle \cite{mixstyle}& ICLR 2021 & 28.2 & 46.7 & 28.1 & 56.6\\
EFDM \cite{efdm}& CVPR 2022 &  29.0&  47.8&  28.5& 57.1 \\
\hline
DSU \cite{dsu}&  ICLR 2022 & 32.0 & 52.0 & 32.4 & 63.7\\
DSU++ &  this paper & \textbf{33.3} & \textbf{54.4} & \textbf{33.6} & \textbf{65.9}\\

\bottomrule[1pt]
\end{tabular}
}
\end{center}
\end{table}

\textbf{Experiment Results:} ReID is a fine-grained instance retrieval task, where the subtle information of persons is important for retrieving an instance. MixStyle and pAdain rely on a reference sample to generate new feature statistics, which might introduce confounded information from the reference sample. Compared to them, our method does better in maintaining the original information and also has more variant possibilities. The experiment results are demonstrated in Table \ref{table:reid}. Our method achieves huge improvement compared to the baseline method and also outperforms the related methods by a big margin.

\begin{table}[h]
\caption{Experiments results of the intra-family and inter-family animal pose estimation performances. Bov., Cerv. and Equ. are short for the animal family of Bovidae, Cervidae, and Equidae, respectively.}
\centering
\label{pose}
\resizebox{1.0\linewidth}{!}{
\begin{tabular}{lc|cc|cc|cc}
\hline
\toprule[1pt]
\multicolumn{2}{l|}{Intra-family}             & \multicolumn{2}{c|}{Bov. to Bov.}  & \multicolumn{2}{c|}{Cerv. to Cerv.} & \multicolumn{2}{c}{Equ. to Equ.}\\ \hline
\multicolumn{1}{l|}{Method}   & Reference  & mAP          & $\text{AP}_{50}$    & mAP          & $\text{AP}_{50}$          & mAP          & $\text{AP}_{50}$   \\ \hline
\multicolumn{1}{l|}{Baseline} & -          & 76.4         & 96.7       & 61.0         & 87.7 & 67.0 & 86.1        \\ 
\multicolumn{1}{l|}{pAdaIN \cite{Nuriel_2021_CVPR}}      & CVPR 2021 & 76.4 & 95.9 & 62.9& 89.3& 67.2&86.2 \\
\multicolumn{1}{l|}{MixStyle \cite{mixstyle}}      & ICLR 2021 & 76.8 & 96.5 & 63.9& 90.0& 68.0 &87.2 \\
\multicolumn{1}{l|}{EFDM \cite{efdm}}      & CVPR 2022 & 77.2 & 96.3 & 63.7& 90.5& 67.5& 87.3\\ \hline
\multicolumn{1}{l|}{DSU \cite{dsu}}      & ICLR 2022  & 77.6         & 97.7              & 64.4        & 90.4   &  68.7 & 87.6      \\ 
\multicolumn{1}{l|}{DSU++}    & this paper &   \textbf{78.1} &  \textbf{98.2}  &\textbf{65.3}         & \textbf{90.9}    & \textbf{69.1} & \textbf{88.2}\\ \toprule[1pt]
\multicolumn{2}{l|}{Inter-family}          & \multicolumn{2}{c|}{Bov. to Cerv.} & \multicolumn{2}{c|}{Bov. to Equ.}   & \multicolumn{2}{c}{Cerv. to Equ.}\\ \hline
\multicolumn{1}{l|}{Method}   & Reference  & mAP          & $\text{AP}_{50}$         & mAP          & $\text{AP}_{50}$      &mAP          & $\text{AP}_{50}$        \\ \hline 
\multicolumn{1}{l|}{Baseline} & -          & 55.8         & 84.3        & 38.4         & 61.7     &  28.9 & 56.5   \\ 
\multicolumn{1}{l|}{pAdaIN \cite{Nuriel_2021_CVPR}}      & CVPR 2021 & 57.9 & 88.2 & 38.4& 62.9& 27.4&56.3 \\
\multicolumn{1}{l|}{MixStyle \cite{mixstyle}}      & ICLR 2021 & 58.6 & 87.4 &41.3 &68.4 & 31.7&59.8 \\
\multicolumn{1}{l|}{EFDM \cite{efdm}}      & CVPR 2022 & 59.5 & 90.2 & 41.2& 68.2 &  32.4& 57.1\\ \hline
\multicolumn{1}{l|}{DSU \cite{dsu}}      & ICLR 2022  & 60.0         & 91.7        & 42.0         & 69.3     & 33.7 & 60.9        \\ 
\multicolumn{1}{l|}{DSU++}    & this paper &           \textbf{61.1}         & \textbf{92.2}           &     \textbf{44.0}         &           \textbf{70.5}  &      \textbf{34.4} & \textbf{61.4}     \\ 
\bottomrule[1pt]
\end{tabular}
}
\end{table}

\subsection{Generalization on Pose Estimation}
\textbf{Setup and Implementation Details:} Pose estimation serves as an important task for vision understanding, which aims to identify the category and position of body keypoint. AP-10K \cite{ap10k} consists of 10015 images from different animal families and species, which is adopted in our experiment to evaluate the model generalization ability of spatial location-based pose estimation task. In this experiment, we adopt ResNet50 as the backbone and choose three mamma families with similar outlooks, \textit{i.e.,} Bovidae, Cervidae, and Equidae, to evaluate the intra-family and inter-family generalization performance.

\textbf{Experiment Results:} The experiment results are summarized in Table \ref{pose}, where we calculate the average score of all species in each family. As can be seen, the model trained on Bovidae shows performance drops when testing in another specie, Cervidae, even they both belong to the artiodactyls order and have similar outlooks. The performance obtains obvious gains equipped with the proposed method. Meanwhile, the intra-family results also demonstrate improvement due to the inherent difference between the different species even in the same family. By combining the inference-time adaptation strategy, the performance can be further improved.
As a result, the experiment result shows our effectiveness of improving generalization ability to both intra-family and inter-family.  

\section{Ablation Study}
In this section, we perform an extensive ablation study of the proposed two plug-and-play modules, including training time uncertainty modeling approach and the inference-time adaptation strategy, on classification (PACS) and segmentation task (GTA5 to Cityscapes) with models trained on ResNet. The effects of different inserted positions and hyper-parameter of the proposed method are analyzed below. 
\subsection{Ablation on Training-time Uncertainty Modeling}
\textbf{Effects of different inserted positions:} DSU can be a plug-and-play module to be readily inserted at any position. Here we name the positions of ResNet after first Conv, Max Pooling layer, 1,2,3,4-th ConvBlock as 0,1,2,3,4,5 respectively. As shown in Table \ref{table:positions}, no matter where the modules are inserted, the performances are consistently higher than the baseline method. The results show that inserting the modules at positions 0-5 would have better performances, which also indicates modeling the uncertainty in all training stages will have better effects. Based on the analysis, we plug the module into positions 0-5 in all experiments. 

\begin{table}[h]
\begin{center}
\setlength{\abovecaptionskip}{0.cm}
\setlength{\belowcaptionskip}{-0pt}
\caption{Effects of different inserted positions during training.}
\label{table:positions}
\resizebox{\linewidth}{!}{
\begin{tabular}{c|c|cccc}
\toprule[1pt]
Inserted Positions  & Baseline     & 0-3 & 1-4  & 2-5  & 0-5 \\
\hline
PACS &  79.0& 82.2 & 83.1 & 83.5 & 84.1 \\
GTA5 to Cityscapes & 37.0 & 41.1 & 40.9& 42.1 & 43.1 \\  
\toprule[1pt]
\end{tabular}
}
\end{center}
\vspace{-2pt}
\end{table}


\textbf{Effects of hyper-parameter:} The hyper-parameter of the probability $p$ is to trade off the strength of feature statistics augmentation. As shown in Figure \ref{fig:prob}, the results are not sensitive to the probability setup and the accuracy reaches the best results when setting $p$ as $0.5$, which is also adopted as the default setting in all experiments.
\begin{figure}[h]
\begin{center}
\setlength{\abovecaptionskip}{0.cm}
\setlength{\belowcaptionskip}{-0.cm}
\includegraphics[width=0.9\linewidth]{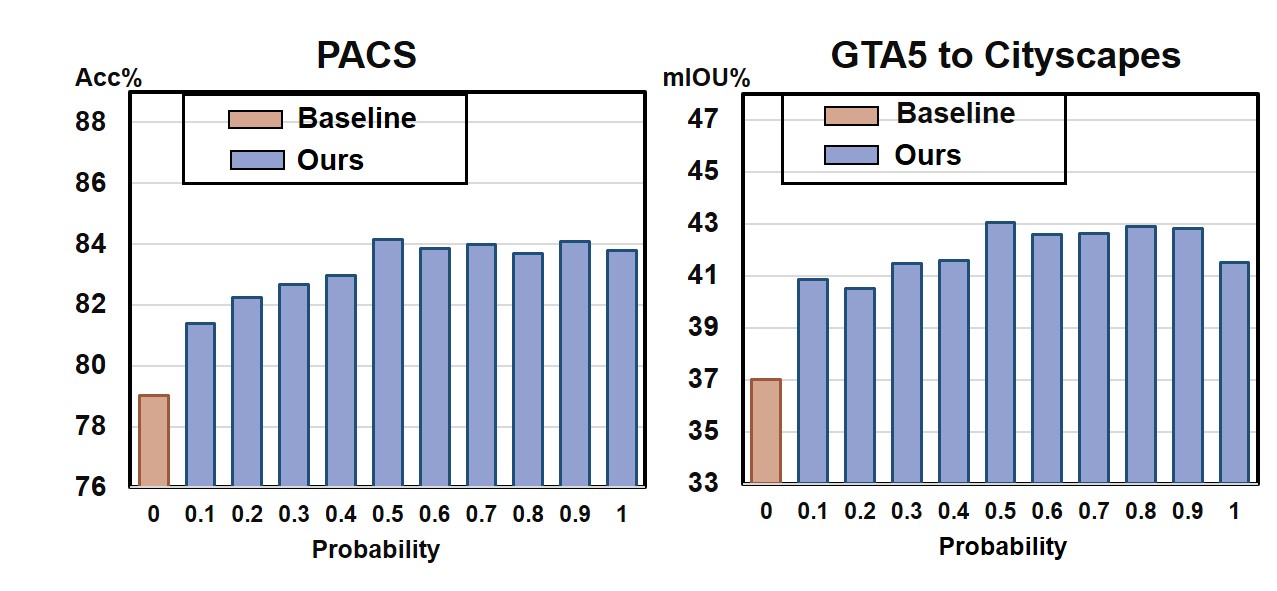}
\caption{The effects on the hyper-parameter probability during training.}
\label{fig:prob}
\end{center}
\vspace{-5pt}
\end{figure}








\subsection{Ablation on Inference-time Adaptation}
\textbf{Effects of different inserted positions:} 
Here we study the different inserted positions in Table \ref{table:positions_tta}. From the results, we can see the inserted position ranges from 0-5 will get better performance, which is consistent with DSU and shows the proposed method can be complementary with the proposed DSU to obtain further improvement.

\begin{table}[h]
\setlength{\abovecaptionskip}{0.cm}
\setlength{\belowcaptionskip}{-0.cm}
\begin{center}
\caption{Effects of different inserted positions of the inference-time adaptation module in DSU++.}
\label{table:positions_tta}
\resizebox{\linewidth}{!}{
\begin{tabular}{c|c|cccccc}
\toprule[1pt]
Inserted   Positions & DSU  & 0-3  & 1-4  & 2-5  & 3-4 & 0-5  \\ \hline
PACS                 & 84.1 & 85.1  & 85.0  &  84.8   &    84.9  &  85.1        \\
GTA5 to Cityscapes      & 43.1  & 43.2 & 43.3 & 43.5 & 43.4  & 43.6 \\

\toprule[1pt]
\end{tabular}}
\end{center}
\end{table}

\textbf{Effects of hyper-parameters:}
As shown in Fig \ref{fig:sigma and weight}, we ablate the calibration weight $\omega$ and the coefficient $n$ 
of shift region in inference-time adaptation. The hyper-parameter $\omega$ is the calibration weight to control the calibration power between the original statistic and the shift region. The statistics remain unchanged when $\omega$ becomes 0, and the statistics are all strictly rectified to the $n\Sigma$ position when $\omega$ becomes 1, which degrades the original statistics information and shows less effective results. Here, we choose $\omega$ to be 0.5 as the default setting, which can provide proper calibration effects as well as maintain its statistics information. 
\begin{figure}[h]
\begin{center}
\includegraphics[width=0.9\linewidth]{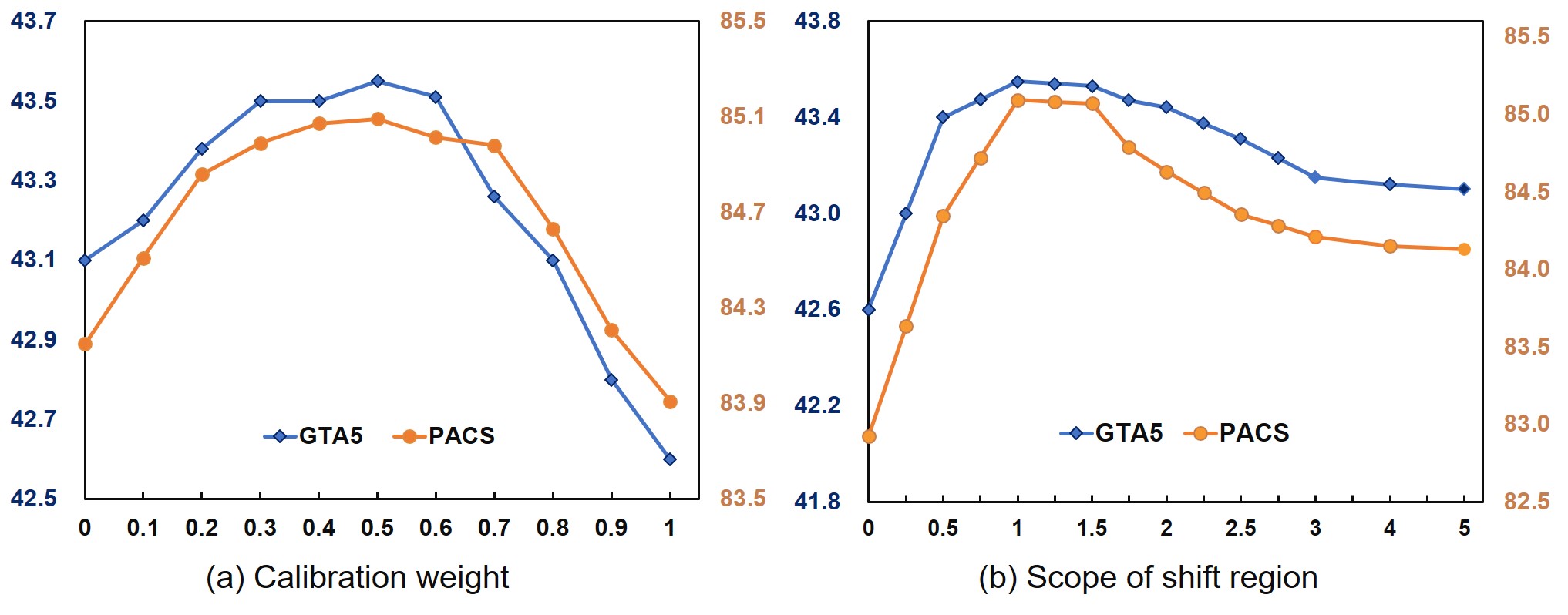}
\setlength{\abovecaptionskip}{0.cm}
\setlength{\belowcaptionskip}{0.cm}
\caption{The effects on the calibration weight $\omega$ (a) and the scope $n$ of the shift region (b).}
\label{fig:sigma and weight}
\end{center}
\vspace{-3pt}
\end{figure}

The hyper-parameter $n$ determines the scope of the shift region for test samples to perform adaptation or not. When the $n$ is larger, some out-of-distribution samples do not get calibrated and the gain will decline. When $n$ becomes 0, all test samples will be calibrated towards the center and the over-calibration might degrade instance-wise information. From the ablation result, we find that when setting $n$ to be 1 performs best, which is set to be the default setting.

\begin{figure*}[h]
\setlength{\abovecaptionskip}{0.cm}
\setlength{\belowcaptionskip}{-0.cm}
\begin{center}
\includegraphics[width=0.9\linewidth]{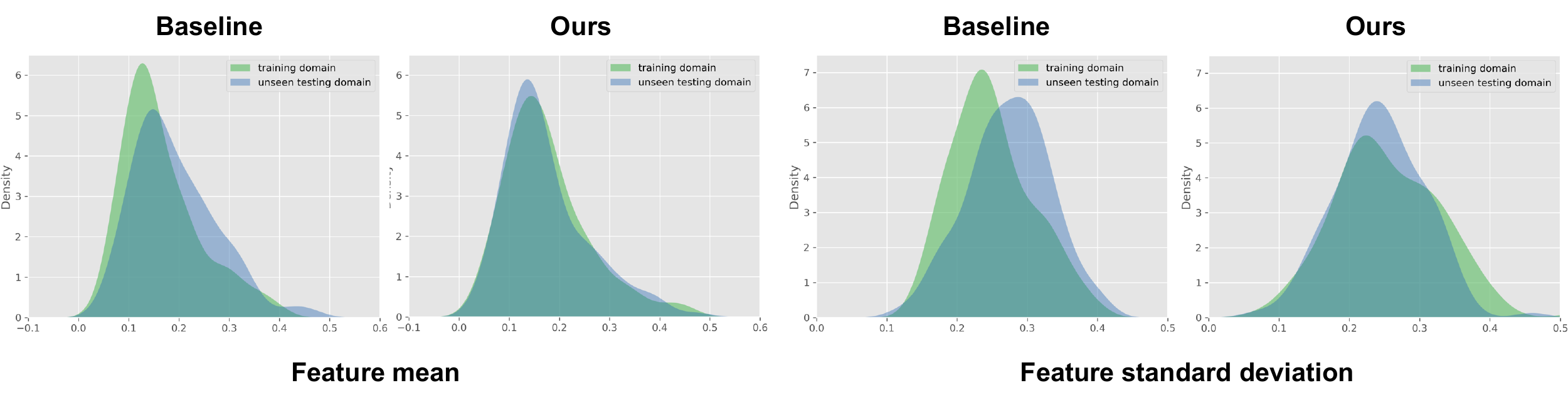}
\caption{Quantitative analysis on the shifts of feature statistics (mean and standard deviation) between training source domains and unseen testing domain, which demonstrates the model can be well trained by modeling uncertainty to gain better robustness against shifts.}
\label{fig:shift}
\end{center}
\vspace{-5pt}
\end{figure*}

\subsection{Ablation on Individual Modules}
We ablate the influences of individual modules in Table \ref{table:individual}. The results show that both the training-time and inference-time modules are both effective in improving the generalization performances. The training-time uncertainty modeling approach benefits the model to gain better generalization ability and thus obtains significant improvement compared to the baseline. Meanwhile, the inference-time strategy further provides an adaptive calibration on instance-wise feature statistics, which enhances the trained model to better deal with various testing shifts. As a result, the training-time uncertainty modeling and inference-time instance adaptation can complement each other and obtain a unified and powerful solution for improving network generalization ability.

\begin{table}[h]
\begin{center}
\setlength{\abovecaptionskip}{0.cm}
\setlength{\belowcaptionskip}{-0pt}
\caption{Effects of individual modules.}
\label{table:individual}
\resizebox{1.0\linewidth}{!}{
\begin{tabular}{l|cc|ccc}
\toprule[1pt]
Methods & \begin{tabular}[c]{@{}l@{}}Uncertainty\\ Modeling\end{tabular} & \begin{tabular}[c]{@{}l@{}}Inference-time\\ Adaptation\end{tabular} & PACS & Seg. & ReID \\
\hline
Baseline &  &  & 79.0 & 37.0 & 25.8 \\
$\text{Baseline}^{\dagger}$ & & \checkmark & 82.3 & 39.8 & 27.1\\  
DSU & \checkmark & & 84.1 & 43.1 & 32.0\\
DSU++ & \checkmark & \checkmark & 85.2 & 43.6 & 33.3 \\
\toprule[1pt]
\end{tabular}
}
\end{center}
\vspace{-3pt}
\end{table}

\section{Quantitative Analysis}\label{analysis}
In this section, we will provide the quantitative analysis of the proposed method on feature statistic discrepancy, domain distance, and semantic changes, respectively. 
\subsection{Robustness to Statistic Shift}
\label{inconsistency}
Quantitative experiments are conducted on PACS, where we choose Art Painting as the unseen testing domain and the rests (\textit{i.e.,} Cartoon, Photo, and Sketch) are used as training domains. To study the phenomena of feature statistic shifts, we capture the intermediate features after the second block in ResNet18 and measure the average feature statistics values of one category in the training and testing domain, respectively. To analyze the model robustness to statistic shift, we show the channel-level feature statistics discrepancy between source training and target testing domain in Fig. \ref{fig:channel_discrepancy}. 

\begin{figure}[h]
\begin{center}
\setlength{\abovecaptionskip}{0.cm}
\setlength{\belowcaptionskip}{-0.cm}
\includegraphics[width=0.95\linewidth]{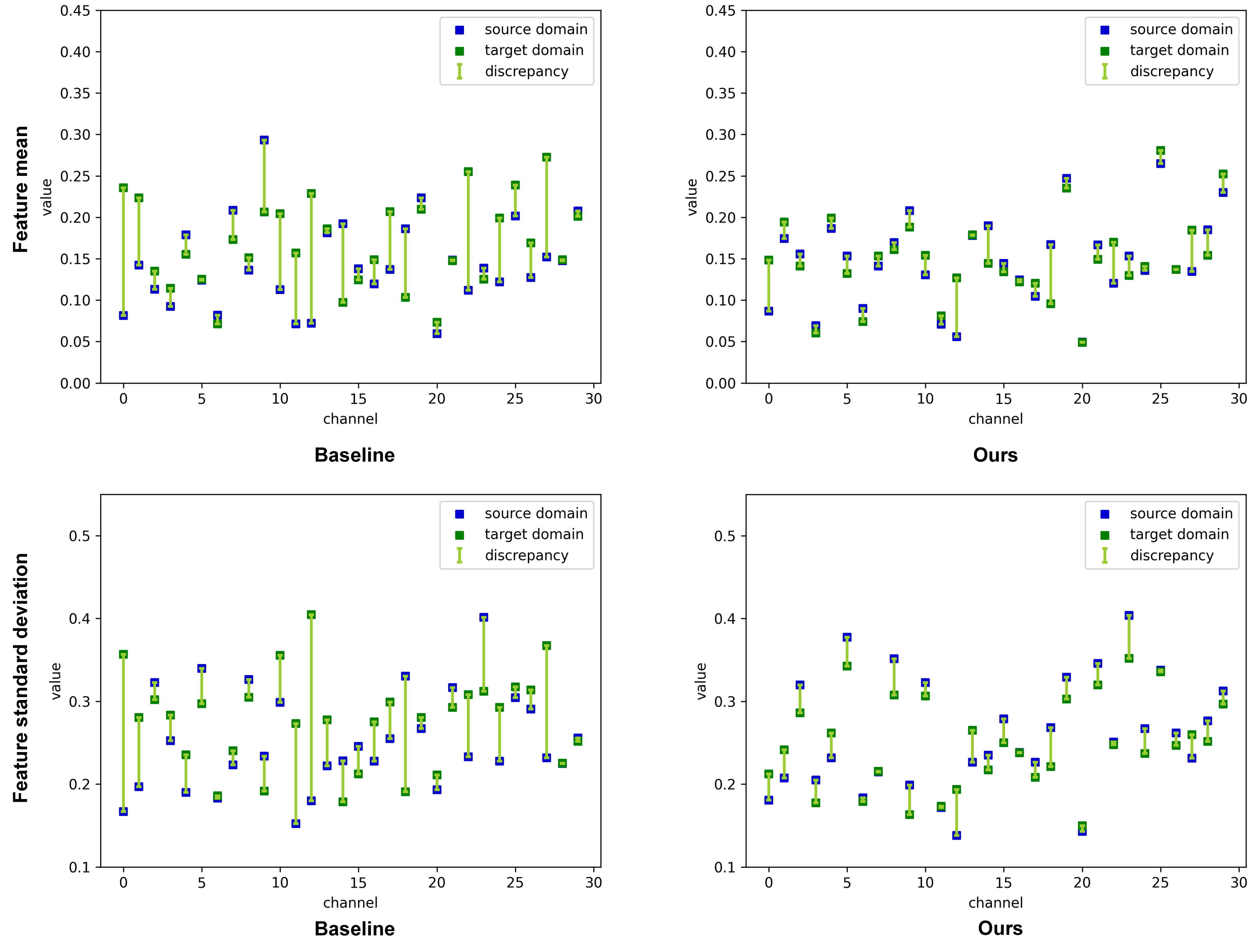}
\caption{The channel-level feature statistic discrepancy comparisons between baseline method and the proposed method DSU. We display the first 30 channels for better visualization.}
\label{fig:channel_discrepancy}
\end{center}
\vspace{-5pt}
\end{figure}

The visualization verifies the previous assumptions. First, it could be seen that our trained model demonstrates an obvious implicit calibration effect on feature statistics, which has a smaller discrepancy to the training domain. Second, the different channels might have different variant scopes and should have different importance to model potential domain shifts, which is consistent with the design of uncertainty estimation. Third, due to the unpredictable property of domain shifts, some channels might have unforeseeable shifts only based on training domain observations and will show an inevitable discrepancy, thus introducing the inference-time adaptation strategy can help the model better deal with various domain shifts. 

The distributions of feature statistics between the source and target domain are shown in Figure \ref{fig:shift}. As the previous works \cite{nips20calibration,nips19trans} show, the feature statistics extracted from the baseline model show an obvious shift due to different data distributions. It can be seen that the model trained with DSU has less shift of the feature statistic distribution. Our method can help the model gain robustness towards domain shifts, as it properly models the potential feature statistic shifts.


\begin{figure*}[h]
\setlength{\abovecaptionskip}{0.cm}
\setlength{\belowcaptionskip}{-0.cm}
\begin{center}
\includegraphics[width=1.0\linewidth]{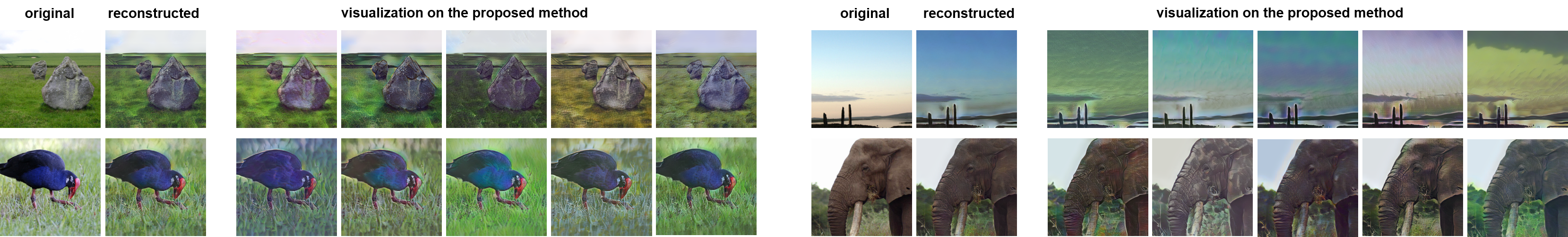}
\caption{The visualization on diverse synthetic changes obtained from our method.
}
\label{fig:visual}
\end{center}
\end{figure*}
\subsection{Domain and Feature Distance}

\begin{table}[h]
\centering
\caption{Quantitative analysis on the intra-domain and inter-domain distance, where wasserstein distance is adopted as the measurement.}
\label{table:was}
\begin{tabular}{l|c|ccc}
\toprule[1pt]
\multicolumn{1}{c|}{\multirow{2}{*}{\begin{tabular}[c]{@{}l@{}}Wasserstein \\ distance\end{tabular}}} & \multirow{2}{*}{\begin{tabular}[c]{@{}l@{}}Between source and \\ target domain\end{tabular}} & \multicolumn{3}{c}{Inner source domains}                             \\ \cline{3-5} 
\multicolumn{1}{c|}{}                                                                       &                                                                                              & \multicolumn{1}{c}{Cartoon} & \multicolumn{1}{c}{Photo} & Sketch \\ \hline
Baseline                                                                                     & 1.22                                                                                         & \multicolumn{1}{c}{0.72}    & \multicolumn{1}{c}{1.04}  & 1.32   \\ 
DSU                                                                                          & 0.77                                                                                         & \multicolumn{1}{c}{0.47}    & \multicolumn{1}{c}{0.63}  & 0.68   \\
\bottomrule[1pt]
\end{tabular}
\end{table}

Besides the effects on feature statistics, we further analyze how the proposed method influences the domain distance in Table \ref{table:was}, \textit{i.e.,} the inner source training domains and between source training and target testing domain. The experiment is conducted on PACS and adopts the wasserstein distance of features as the measurement. It can be seen that the distances between individual training domains and the entire training domain are narrowed, verifying the conclusion that DSU can help the model encode better domain-invariant features. At the same time, the distance between the training and test domain is largely narrowed from 1.22 to 0.77, it verifies our theoretical analysis that modeling the statistic uncertainty can improve the robustness of domain shifts and implicitly decrease the domain distance under domain shifts.

\begin{figure}[ht]
\setlength{\abovecaptionskip}{0.cm}
\setlength{\belowcaptionskip}{-0.cm}
\begin{center}
\includegraphics[width=0.4\textwidth]{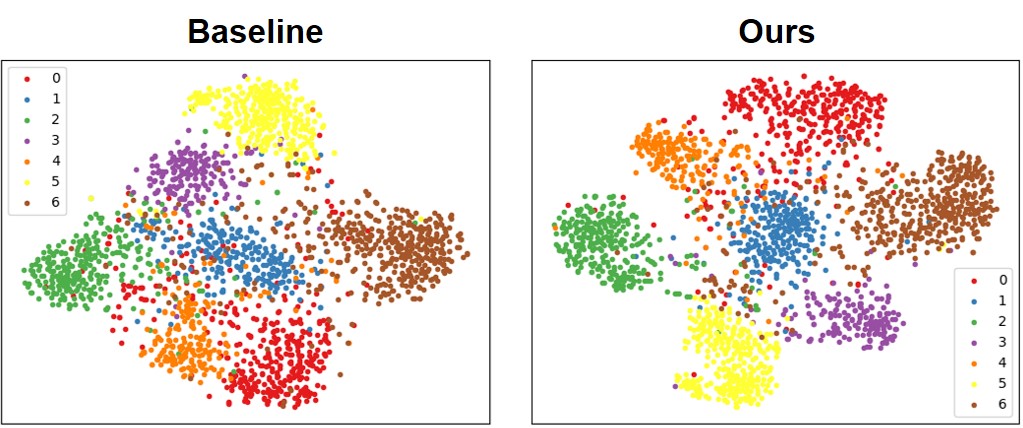}
\caption{The t-SNE visualization of different classes on unseen PACS domain.}
\label{fig:tsne}
\end{center}
\vspace{-2pt}
\end{figure}

To analyze the effects on feature representations, we visualize the feature representation vectors of different categories in the unseen domain with t-SNE \cite{tsne} in Figure \ref{fig:tsne}. The features of the same category become more compact and the classification boundary becomes more distinct, benefiting from the proposed method. Because our method can alleviate the domain perturbations during training and make the model focus on content information, obtaining more invariant feature representations.


\subsection{Visualization on the Synthetic Changes}

Besides the quantitative experiment results, we also obtain a more intuitional view of the diverse changes provided by our method, through visualizing the reconstruction results using a predefined autoencoder \cite{adain}, where the proposed module is inserted into the encoder, and inverse the feature representations into synthetic images after the decoder. As the results shown in Figure \ref{fig:visual}, the reconstructed images obtained from our probabilistic approach show diverse synthetic changes, such as the environment, object texture, and contrast, etc. It shows that modeling the feature statistics as uncertain distributions can provide diverse possibilities of domain perturbations.

\section{Discussion}

\subsection{Within-dataset Performance}

In Table \ref{table:within}, we provide the within-dataset performance on PACS. According to its multi-source training protocol, the within-domain performance is averaged over multiple training-domain datasets (P,A,C,S denotes Art, Cartoon, Photo, and Sketch respectively). As shown in Table \ref{table:within}, it can be observed that the within-dataset performance of our DSU still outperforms the baseline, indicating that DSU does not sacrifice the in-domain performance to gain benefits on out-of-distribution domains. The reason might be that instances in the testing set may not always fall into exactly the same distribution of the training set, and they still have slight statistic shifts \cite{rbn}. The proposed DSU can help the trained model improve the robustness to statistics shifts and thus gain better performance in within-dataset. Besides, the experiment results on the intra-family pose estimation task also verify the conclusion.

\begin{table}[h]

\small
\center{
\caption{Within-dataset performance on PACS. P,A,C,S denote Photo, Art Painting, Cartoon, and Sketch respectively.}
\label{table:within}
\resizebox{\linewidth}{!}{
\begin{tabular}{l|c|cccc|c}
\toprule[1pt]
\multicolumn{1}{l|}{Method} & Reference & P,C,S & P,A,S & A,C,S & P,A,C & Average (\%)  \\ \hline
Baseline&- &95.70 &95.28  &94.22  &96.58  & 95.44\\

DSU& Ours & 96.20 & 96.32  & 95.17 & 97.20 &96.21 \\

\bottomrule[1pt]
\end{tabular}
}
}
\end{table}


\subsection{Choice of Uncertainty Distribution}

In our method, the Gaussian distribution with uncertainty estimation is adopted as the default setting, we also conduct other distributions for comparisons in Table \ref{table:fixed}. Specifically, Random denotes directly adding random shifts drawn from a fixed Gaussian $\mathcal{N}(0,1)$, and Uniform denotes that the shifts are drawn from $\text{U}(-\Sigma,\Sigma)$, where $\Sigma$ is the scope obtained from our uncertainty estimation. As we can see, directly using Gaussian distribution with the improper variant scope will harm the model performances, indicating the variant range of feature statistics should have some instructions. Further analysis about different vanilla Gaussian distributions with pre-defined standard deviation is conducted in the Appendix. Meanwhile, the result of Uniform shows some improvement but is still lower than DSU, which indicates the boundless Gaussian distribution is more helpful to model more diverse variants.

\begin{table}[h]
\centering
\caption{Intensive study about different vanilla Gaussian distributions with pre-defined standard deviation.}
\resizebox{\linewidth}{!}{
\begin{tabular}{c|c|cccccc}
\toprule[1pt]
Choice  & Baseline & Rand(${10^{0}}$) & Rand(${10^{\text{-1}}}$) &  Rand($10^{\text{-2}}$) & Rand($10^{\text{-3}}$) & DSU \\
\hline
PACS &  79.0& 76.9 & 81.2 & 79.3 & 79.1  &84.1 \\
GTA5 to Cityscapes & 37.0 & 38.2 & 39.8 & 40.1 & 38.9 & 43.1\\  
\toprule[1pt]
\end{tabular}
}
\label{table:fixed}
\end{table}

Besides the analysis of the uncertainty estimation in the ablation study, we also conduct a more intensive study about the effects of pre-defined uncertainty estimations. Specifically, Rand$(s)$ denotes directly imposing random shifts drawn from a fixed Gaussian $\mathcal{N}(0,s^2)$. The intensive study is shown in Table \ref{table:fixed}. It can be observed that the results of different fixed distributions are all much lower than the proposed method. Some conclusions could be obtained from the results. (a): Imposing excessive uncertainty might harm the model training and degrade the performance. (b): The best fixed value of uncertainty estimation might vary from different tasks. By contrast, the proposed method can be adaptive to different tasks without any manual adjustment. 

\begin{table}[h]
\centering

\caption{Study about the effects of sharing the same uncertain distribution among different channels.}
\begin{tabular}{c|c|cc}
\toprule[1pt]
Choice  & Baseline & Channel-share & DSU \\
\hline
PACS &  79.0& 80.2 &84.1 \\
GTA5 to Cityscapes & 37.0 & 39.3 & 43.1\\  
\toprule[1pt]
\end{tabular}
\label{table:same}
\end{table}

We also conduct the experiment to test the effectiveness of treating different channels with different potentials. Channel-share denotes all channels of the sample share the same uncertainty distribution, \textit{i.e.,} using the average uncertainty estimation among channels. As shown in Table \ref{table:same}, the results indicate that sharing the same uncertain distribution among different channels is less effective, which ignores the different potentials of channels and will limit their performances. Meanwhile, the proposed method explicitly considers the different potentials of different channels and brings better performances.

\subsection{Efficiency of Our Method}
DSU++ has manifested obvious performance gains in various vision tasks compared to the baseline method, without bells and whistles. To show its efficiency, we provide the comparisons of running time and additional parameters to baseline in Table \ref{table:efficiency}. DSU is a non-parametric method and does not provide trainable parameters. The training domain channel-wise shift regions are required during inference-time, which slightly increases $\sim 0.02\%$ additional parameter compared to baseline. From the table, we can see that the training and inference time increase by $\sim 6\%$ and $\sim 1\%$ to the baseline method, which is a relatively small afford to runtime and thus is very efficient.

\begin{table}[h]

\small
\center{
\caption{The comparisons of model parameters and running time.}
\label{table:efficiency}
\resizebox{\linewidth}{!}{
\begin{tabular}{l|c|ccc}
\toprule[1pt]
Method & \multicolumn{1}{c|}{Reference} & \multicolumn{1}{c}{\begin{tabular}[c]{@{}c@{}}Model parameter\\ (number)\end{tabular}} & \multicolumn{1}{c}{\begin{tabular}[c]{@{}c@{}}Training time\\ (millisecond)\end{tabular}} & \multicolumn{1}{c}{\begin{tabular}[c]{@{}c@{}}Inference time\\ (millisecond)\end{tabular}} \\
\hline
Baseline&- &111,775,12 &8.03 ms  & 2.81 ms\\
DSU++& Ours & 111,798,16 & 8.41 ms  & 2.84 ms  \\


\bottomrule[1pt]
\end{tabular}
}
}
\end{table}

\section{Conclusions}
\label{others}
In this paper, we propose a probabilistic approach to improve the network generalization ability by modeling the uncertainty of domain shifts with synthesized feature statistics during training. Each feature statistic is hypothesized to follow a multi-variate Gaussian distribution for modeling the diverse potential shifts. Due to the generated feature statistics with diverse distribution possibilities, the models can gain better robustness towards diverse domain shifts. During inference, we propose to adaptively calibrate feature statics to enhance the model generalization ability to deal with various domain shifts. Comprehensive experiment results and theoretical analysis demonstrate the effectiveness of our method in improving the network generalization ability.
\ifCLASSOPTIONcompsoc
  \section*{Acknowledgments}
  \label{acknowlegement}
This work was supported by the National Natural Science Foundation of China under Grant 62088102, and in part by the PKU-NTU Joint Research Institute (JRI) sponsored by a donation from the Ng Teng Fong Charitable Foundation. We sincerely appreciate Dr. Ying Shan for providing valuable GPU computing resource and insightful comments to improve the presentation of this work.
\fi

\ifCLASSOPTIONcaptionsoff
  \newpage
\fi

\bibliographystyle{IEEEtran}
\bibliography{IEEEabrv,pami}

\clearpage
\onecolumn
\appendices
\section{Theory Proof}
\label{sec: appendix}
We recall the notation that we use in the main paper as well as this appendix.

\textbf{Supplementary Notation} 
$\circ$ denotes composition. $\mathbbm{1}_{a\times b}$ denotes the vector of size $a\times b$ with all components equal one. $\mathbf{I}$ denotes the identity matrix. The subscript $_{||f||_L\leq 1}$ denotes the set of all Lipschitz continuous functions with a Lipschitz constant less than 1. $\langle x, y\rangle$ denotes the cosine metric. $(\cdot)_i$ denotes the i-th channel. The expected value of a random variable $X$ is denoted by $\mathbb{E}[X]$. $\odot$ denotes Hadamard product. $\cdot$ denotes both dot product and scalar multiplication. We give some examples of these operations as follows.

If $k$ is a constant, $x=(x_{ij})_{a\times b},y=(y_{ij})_{a\times b}$ are two tensors of size $a\times b$ and $z=(z_{ij})_{b \times c}$ is a tensor of size $b\times c$, then
\begin{equation}
\begin{split}
     x\odot y=(x_{ij}y_{ij}&)_{a\times  b}, \ \frac{x}{y}=(\frac{x_{ij}}{y_{ij}})_{a\times b}, \\
     k\cdot x=(k\,x_{ij})_{a\times b}&, \ x\cdot y=\sum_{i=1}^{a}\sum_{j=1}^{b}x_{ij}y_{ij}, \\
     \langle x,y \rangle = \frac{x\cdot y}{||x||_2||y||_2}, \ (x)_i=&(x_{ij})_{1\times b}, \ xz=(\sum_{t=1}^{b}x_{it}z_{tj})_{a\times c}.
\end{split}
\end{equation}
\subsection{Proof of Theorem 4.1}
\label{Appendix: generalization error bound}
In Theorem \ref{Generalization bound}, we bound the generalization gap by employing Rademacher Complexity, which is a measure of the richness of function classes. And we give the definition as follows:
\begin{definition}
\label{rademacher}
\textbf{(Rademacher Complexity\cite{bartlett2002rademacher})} Given a space $\mathcal{X}$ and a fixed distribution $\mathcal{D}$ over $\mathcal{X}$. $X = \{x_1$, . . . , $x_n\}$ is a set of samples drawn from $\mathcal{D}$-i.i.d. Let $\mathcal{F}$ be a class of real-valued function $f:\mathcal{X}\to R$. The Empirical Rademacher Complexity of the function class $F$ is defined to be 
\begin{equation}
\hat{\mathcal{R}}_X(\mathcal{F})=\mathbbm{E}_{\sigma}[\mathop{sup}\limits_{f\in\mathcal{F}}(\frac{1}{n}\sum\limits_{i=1}^{n}\sigma_i f(x_i)) ].
\end{equation}
where $\sigma_1$, . . . , $\sigma_n$ are independent random variables uniformly chosen from $\{ -1, 1 \}$. \\
And the Rademacher Complexity of the function class $F$ is defined to be
\begin{equation}
\mathcal{R}_n(\mathcal{F})=\mathbbm{E}_{X\sim\mathcal{D}^n}[\hat{\mathcal{R}}_X(\mathcal{F})].
\end{equation}
\end{definition}
We need to use an inequality about Rademacher Complexity in our proof and we take it as a lemma.
\begin{lemma}
\label{rademacher inequality}
\textbf{(Rademacher Complexity bound \cite{bousquet2003introduction})} 
Given a space $\mathcal{X}$ and a fixed distribution $\mathcal{D}$ over $\mathcal{X}$. $X=\{x_1$, . . . , $x_n\}$ is a set of samples drawn from $\mathcal{D}$-i.i.d. Let $\mathcal{F}$ be a class of real-valued function $f:\mathcal{X}\to [0,1]$. The following inequality holds for every function $f\in \mathcal{F}$ with probability at least $1-\delta$$:$
\begin{equation}
    \mathbbm{E}_{\mathcal{D}}[f(x)]\leq \frac{1}{n}\sum\limits_{i=1}^{n}f(x_i)+2\mathcal{R}_n(\mathcal{F})+\sqrt{\frac{log(\frac{1}{\delta})}{2n}}.
\end{equation}
In addition, the following inequality holds for every function $f\in \mathcal{F}$ with probability at least $1-\delta$$:$
\begin{equation}
    \mathbbm{E}_{\mathcal{D}}[f(x)]\leq \frac{1}{n}\sum\limits_{i=1}^{n}f(x_i)+2\hat{\mathcal{R}}_X(\mathcal{F})+3\sqrt{\frac{log(\frac{2}{\delta})}{2n}}.
\end{equation}
\end{lemma}
We require to prove another important lemma before we prove the main theorems of our work.
\begin{lemma}
\label{Appendix: lemma}
Consider a $L$-Lipschitz function class $\mathcal{H}$  and two probability distribution $\mathcal{D}_1,\mathcal{D}_2$
over space $\mathcal{X}$.  The following inequality holds for the Wasserstein distance between $\mathcal{D}_1,\mathcal{D}_2$:
\begin{equation}
    \mathcal{W}(\mathcal{D}_1,\mathcal{D}_2)\geq \frac{1}{2L} |\epsilon_{\mathcal{D}_1}(h_1, h_2)- \epsilon_{\mathcal{D}_2}(h_1, h_2)|,
\end{equation}
for any $h_1,h_2\in \mathcal{H}$.
\end{lemma}
\begin{proof}
As any function in $\mathcal{H}$ is $L$-Lipschitz, we have
\begin{equation}
    \begin{split}
    |h_1(x)-h_1(y)|\leq L\cdot ||x-y||, \forall x,y\in\mathcal{X}. \\
    |h_2(x)-h_2(y)|\leq L\cdot ||x-y||, \forall x,y\in\mathcal{X}.
    \end{split}
\end{equation}
From the triangle inequality, we have
\begin{equation}
\label{Appendix: triangle}
    |(h_1(x) - h_2(x)) - (h_1(y) - h_2(y))|\leq|h_1(x)-h_1(y)|+|h_2(x)-h_2(y)| \leq 2L\cdot ||x-y||,\forall x,y\in \mathcal{X}.
\end{equation}
The inequality \ref{Appendix: triangle} shows that the function $h_1- h_2$ is $2L$-Lipschitz for any $h_1,h_2\in\mathcal{H}$. From the Kantorovich-Rubinstein Duality, we have 
\begin{equation}
\label{Appendix: Duality}
    \mathcal{W}(\mathcal{D}_1,\mathcal{D}_2)=sup_{||f||_L\leq 1}\mathbb{E}_{x\sim \mathcal{D}_1}[f(x)]-\mathbb{E}_{x\sim \mathcal{D}_2}[f(x)].
\end{equation}
Combining the Equation \ref{Appendix: triangle} and \ref{Appendix: Duality}, we have
\begin{equation}
\begin{split}
 \mathcal{W}(\mathcal{D}_1,\mathcal{D}_2)&=sup_{||f||_L\leq 1}\mathbb{E}_{x\sim \mathcal{D}_1}[f(x)]-\mathbb{E}_{x\sim \mathcal{D}_2}[f(x)] \\
 &\geq \frac{1}{2L} |\mathbb{E}_{z\sim \mathcal{D}_1}[|h_1(z)-h_2(z)|]-\mathbb{E}_{z\sim \mathcal{D}_2}[|h_1(z)-h_2(z)|]\,| \\
 &= \frac{1}{2L}|\epsilon_{\mathcal{D}_1}(h_1, h_2)- \epsilon_{\mathcal{D}_2}(h_1, h_2)|.
\end{split}
\end{equation}
\end{proof}

\noindent We now prove the Theorem \ref{Generalization bound}.
\begin{theorem}
\textbf{(Generalization gap bound)} Consider the representation function $h_k:\mathcal{X}\rightarrow \mathcal{Z}_k$ and a Lipschitz continuous function class $\mathcal{H}$ with Lipschitz constant $L$. If $X = \{x_1$, . . . , $x_n\}$ is a set of samples drawn from $\mathcal{S}$-i.i.d labeled according to $f$. Then, for any $h \in \mathcal{H}$, the following bound holds with probability at least $1-\delta$$:$  
\begin{equation}
    \begin{split}
        \epsilon_{\mathcal{T}}(h)&\leq \hat{\epsilon}_{\mathcal{S}}(h)+2L\sum\limits_{i=1}^{S}\Pi_{i}\cdot \int_{e} \Pi_e\cdot \mathcal{W}(\widetilde{\mathcal{D}}_{e_i},\widetilde{\mathcal{D}}_e)   \,\mathsf{d}e +\lambda+2\hat{\mathcal{R}}_X(\mathcal{H'})+3\sqrt{\frac{log(\frac{2}{\delta})}{2n}}.
    \end{split}
\end{equation}
where $\hat{\epsilon}_{\mathcal{S}}(h)=\frac{1}{n}\sum\limits_{i=1}^n |h(h_k(x_i))-f(x_i)|$ is the empirical risk and $\mathcal{H'}:=\{|h\circ h_k - f|\ |h\in \mathcal{H}\}$.
\end{theorem}
\begin{proof}
Let $h^*=\text{arg}\mathop\text{min}_{h\in \mathcal{H}}(\epsilon_{\mathcal{S}}(h, f_k)+\epsilon_{\mathcal{T}}(h,f_k))$. According to the Definition \ref{close}, we have $\epsilon_{\mathcal{S}}(h^*, f_k)+\epsilon_{\mathcal{T}}(h^*, f_k)\leq \lambda$.
\begin{equation}
\label{Appendix: proof2}
    \begin{split}
        \epsilon_{\mathcal{T}}(h)&\leq \epsilon_{\mathcal{T}}(h^*,f_k)+\epsilon_{\mathcal{T}}(h^*,h)\\
        &\leq \epsilon_{\mathcal{T}}(h^*,f_k)+\epsilon_{\mathcal{S}}(h^*,h)+|\epsilon_{\mathcal{T}}(h^*,h)-\epsilon_{\mathcal{S}}(h^*,h)|\\
        &\leq \epsilon_{\mathcal{T}}(h^*,f_k)+\epsilon_{\mathcal{S}}(h^*,h)+ \sum_{i=1}^{S}\Pi_i\cdot\int_{e}\Pi_e\cdot |\epsilon_{\mathcal{D}_{e}}(h^*,h)-\epsilon_{\mathcal{D}_{e_i}}(h^*,h)|\,\mathsf{d}e \\
        &\leq \epsilon_{\mathcal{T}}(h^*,f_k)+\epsilon_{\mathcal{S}}(h^*,h)+      2L\sum_{i=1}^{S}\Pi_i\cdot\int_{e}\Pi_e\cdot \mathcal{W}(\widetilde{\mathcal{D}}_{e_i},\widetilde{\mathcal{D}}_e)   \,\mathsf{d}e \\
        &\leq \epsilon_{\mathcal{T}}(h^*,f_k)+\epsilon_{\mathcal{S}}(h^*,f_k)+\epsilon_{\mathcal{S}}(h,f_k)+2L\sum_{i=1}^{S}\Pi_i\cdot\int_{e}\Pi_e\cdot \mathcal{W}(\widetilde{\mathcal{D}}_{e_i},\widetilde{\mathcal{D}}_e)   \,\mathsf{d}e \\
        &\leq \lambda+\epsilon_{\mathcal{S}}(h,f_k)+2L\sum_{i=1}^{S}\Pi_i\cdot\int_{e}\Pi_e\cdot \mathcal{W}(\widetilde{\mathcal{D}}_{e_i},\widetilde{\mathcal{D}}_e)   \,\mathsf{d}e.
    \end{split}
\end{equation}
Note that, functions in $\mathcal{H'}$ are mappings from $\mathcal{X}$ to $[0,1]$. Through the lemma \ref{rademacher inequality}, we can bound the real risk $\epsilon_{\mathcal{S}}(h,f_k)$ by its empirical result $\hat{\epsilon}_{\mathcal{S}}(h,f_k)$. 
\begin{equation}
\label{Appendix: proof3}
    \epsilon_{\mathcal{S}}(h,f_k)\leq \hat{\epsilon}_{\mathcal{S}}(h,f_k)+2\hat{\mathcal{R}}_X(\mathcal{H'})+3\sqrt{\frac{log(\frac{2}{\delta})}{2n}}.
\end{equation}
Combining  Equation \ref{Appendix: proof2} and \ref{Appendix: proof3}, we see that the final result is proved.
\end{proof}
\subsection{Proof of Theorem 4.2}
In section \ref{theory}, we state a simplified version of the theorem about implicit regularization according to torch's computational rules for convenience. Here we rewrite it into a formal mathematical format and prove it.
\begin{theorem}
\label{Appendix: Implicit regularization}
Consider the case of DSU for binary classification with a family $\mathcal{F}$ of linear functions $f_{w, b}(x) = w\cdot x+b$, where $x\in R^{C\times N}$ is the output of k-th layer in network, $w\in R^{C\times N}$ and $b\in R$ are coefficients, $w\cdot x$ represents the operation of dot product. The original loss function is $R(f) = \frac{1}{n}\sum_{i=1}^{n}(f(x_i)-y_i)^2 $. The expectation of loss function after augmentation in k-th layer is:
\begin{equation}
\begin{split}
\label{Appendix: expectation}
    \mathbb{E}[R(f)] &= \frac{1}{n}\sum_{i=1}^{n}\mathbb{E}_{\epsilon_{\sigma},\epsilon_{\mu}}[(w\cdot ((((\sigma(x_i)+ \epsilon_{\sigma}\odot{\Sigma}_{\sigma}(x)) \mathbbm{1}_{1\times N})\odot\frac{x_i-\mu(x_i)\mathbbm{1}_{1\times N}}{\sigma(x_i)\mathbbm{1}_{1\times N}})+\mu(x_i)\mathbbm{1}_{1\times N} \\
    &+(\epsilon_{\mu}\odot{\Sigma}_{\mu}(x))\mathbbm{1}_{1\times N})+b-y_i)^2]. \\
\end{split}
\end{equation}
And we can prove that it equals to the following equation:
\begin{equation}
\begin{split}
    \mathbb{E}[R(f)] &= \frac{1}{n}\sum_{i=1}^{n}(f(x_i)-y_i)^2+ 
    \sum_{i=1}^{C} (\Sigma_{\mu})_i^2\cdot ||w_i||_2^{2} +
    \frac{1}{n}\sum_{i=1}^{C}(\Sigma_{\sigma})_i^2\cdot\sum_{j=1}^{n}||w_i||_2^{2}\cdot \langle w_i,(\frac{x_j- \mu (x_j)\mathbbm{1}_{1\times N}}{\sigma (x_j)\mathbbm{1}_{1\times N}})_i\rangle^2.
\end{split}
\end{equation}
where $\Sigma_{\sigma}$, $\Sigma_{\mu}$, $\epsilon_{\sigma}$, $\epsilon_{\mu}$ are same as $\Sigma_{\sigma}(x)$, $ \Sigma_{\mu}(x)$, $\epsilon_{\sigma}$, $\epsilon_{\mu}$ defined in Equation (\ref{Sigmamu}, \ref{Sigmasigma}, \ref{dsu_training1}, \ref{dsu_training2}).
\end{theorem}
\begin{proof}
For convenience, we substitute $z_i$ for $\frac{x_i-\mu(x_i)\mathbbm{1}_{1\times N}}{\sigma(x_i)\mathbbm{1}_{1\times N}}$ in the following proof. And we first make some simplifications to Equation \ref{Appendix: expectation}:
\begin{equation}
    \begin{split}
    \label{transform1}
        \mathbb{E}[R(f)] &= \frac{1}{n}\sum_{i=1}^{n}\mathbb{E}_{\epsilon_{\sigma},\epsilon_{\mu}}[(w\cdot(x_i-\mu(x_i)\mathbbm{1}_{1\times N}+((\epsilon_{\sigma}\odot{\Sigma}_{\sigma}(x)) \mathbbm{1}_{1\times N})\odot z_i+\mu(x_i)\mathbbm{1}_{1\times N}+(\epsilon_{\mu}\odot{\Sigma}_{\mu}(x))\mathbbm{1}_{1\times N})+b-y_i)^2] \\
        &=\frac{1}{n}\sum_{i=1}^{n}\mathbb{E}_{\epsilon_{\sigma},\epsilon_{\mu}}[(w\cdot(((\epsilon_{\sigma}\odot{\Sigma}_{\sigma}(x)) \mathbbm{1}_{1\times N})\odot z_i+(\epsilon_{\mu}\odot{\Sigma}_{\mu}(x))\mathbbm{1}_{1\times N}) +w\cdot x_i+b-y_i)^2] \\
        &=\frac{1}{n}\sum_{i=1}^{n}(w\cdot x_i+b-y_i)^2+\frac{1}{n}\sum_{i=1}^{n}\mathbb{E}_{\epsilon_{\sigma},\epsilon_{\mu}}[(w\cdot(((\epsilon_{\sigma}\odot{\Sigma}_{\sigma}(x)) \mathbbm{1}_{1\times N})\odot z_i+(\epsilon_{\mu}\odot{\Sigma}_{\mu}(x))\mathbbm{1}_{1\times N}))^2 ]\\
        &+\frac{2}{n}\sum_{i=1}^{n}\mathbb{E}_{\epsilon_{\sigma},\epsilon_{\mu}}[(w\cdot(((\epsilon_{\sigma}\odot{\Sigma}_{\sigma}(x)) \mathbbm{1}_{1\times N})\odot z_i+(\epsilon_{\mu}\odot{\Sigma}_{\mu}(x))\mathbbm{1}_{1\times N}))(w\cdot x_i+b-y_i)].
    \end{split}
\end{equation}
Since the $\epsilon_{\sigma}, \epsilon_{\mu}$ are standard Gaussian distributions, the first-order terms in the Equation \ref{transform1} are all equal to zero:
\begin{equation}
\begin{split}
\label{transform2.1}
\sum_{i=1}^{n}\mathbb{E}_{\epsilon_{\sigma},\epsilon_{\mu}}[(w\cdot((\epsilon_{\mu}\odot{\Sigma}_{\mu}(x))\mathbbm{1}_{1\times N}))(w\cdot x_i+b-y_i)]=0, 
\end{split}
\end{equation}
\begin{equation}
\begin{split}
\label{transform2.2}
\sum_{i=1}^{n}\mathbb{E}_{\epsilon_{\sigma},\epsilon_{\mu}}[(w\cdot(((\epsilon_{\sigma}\odot{\Sigma}_{\sigma}(x)) \mathbbm{1}_{1\times N})\odot z_i)(w\cdot x_i+b-y_i)]=0.
\end{split}
\end{equation}
Since the $\epsilon_{\sigma}, \epsilon_{\mu}$ are independent distributions, the second-order term consisting of two first-order terms is equal to zero:
\begin{equation}
\begin{split}
\label{transform3}
&\quad\, \sum_{i=1}^{n}\mathbb{E}_{\epsilon_{\sigma},\epsilon_{\mu}}[(w\cdot(((\epsilon_{\sigma}\odot{\Sigma}_{\sigma}(x)) \mathbbm{1}_{1\times N})\odot z_i))(((\epsilon_{\mu}\odot{\Sigma}_{\mu}(x))\mathbbm{1}_{1\times N})(w\cdot x_i+b-y_i))] \\
&=\sum_{i=1}^{n}\mathbb{E}_{\epsilon_{\sigma}}[(w\cdot(((\epsilon_{\sigma}\odot{\Sigma}_{\sigma}(x)) \mathbbm{1}_{1\times N})\odot z_i)]\mathbb{E}_{\epsilon_{\mu}}[((\epsilon_{\mu}\odot{\Sigma}_{\mu}(x))\mathbbm{1}_{1\times N}))(w\cdot x_i+b-y_i)] \\
&=0.
\end{split}
\end{equation}
Combining Equation \ref{transform1}, \ref{transform2.1},\ref{transform2.2} and \ref{transform3}, we have
\begin{equation}
\begin{split}
\label{transform4}
    \mathbb{E}[R(f)]&=\frac{1}{n}\sum_{i=1}^{n}\mathbb{E}_{\epsilon_{\sigma}\sim\mathcal{N}(\textbf{0},\mathbf{I})}[(w\cdot(((\epsilon_{\sigma}\odot{\Sigma}_{\sigma}(x)) \mathbbm{1}_{1\times N})\odot z_i))^2]+\mathbb{E}_{\epsilon_{\mu}\sim\mathcal{N}(\textbf{0},\mathbf{I})}[((\epsilon_{\mu}\odot{\Sigma}_{\mu}(x))\mathbbm{1}_{1\times N})^2 ]\\
        &+\frac{1}{n}\sum_{i=1}^{n}(f(x_i)-y_i)^2.
\end{split}
\end{equation}
Note that, each channel of $z_i=\frac{x_i-\mu(x_i)\mathbbm{1}_{1\times N}}{\sigma(x_i)\mathbbm{1}_{1\times N}}$ is an unit vector:
\begin{equation}
\begin{split}
     \text{Denote}\ x_i:=(u_{st})_{C\times N},\forall s\in\{1,...,N\}, \quad &||(z_i)_s||_2=\sqrt{\sum_{t=1}^{N}(\frac{u_{st}-(\mu(x_i))_s}{(\sigma(x_i))_s})^2}=1,\\
     \text{where}\   (\mu(x_i))_s=\frac{1}{N}\sum_{t=1}^N u_{st}, \quad (\sigma&(x_i))_s=\sqrt{\frac{1}{N}\sum_{t=1}^{N}(u_{st}-(\mu(x_i))_s)^2}.
\end{split}    
\end{equation}
Then, we have
\begin{equation}
\begin{split}
\label{transform5}
&\quad\, \frac{1}{n}\sum_{i=1}^{n}\mathbb{E}_{\epsilon_{\sigma}\sim\mathcal{N}(\textbf{0},\textbf{I})}[(w\cdot(((\epsilon_{\sigma}\odot{\Sigma}_{\sigma}(x)) \mathbbm{1}_{1\times N})\odot z_i))^2] \\
&=\frac{1}{n}\sum_{i=1}^{n}\mathbb{E}_{\epsilon_{\sigma}\sim\mathcal{N}(\textbf{0},\textbf{I})}[(w\odot w)\cdot ((\epsilon_{\sigma}\odot \epsilon_{\sigma}\odot {\Sigma}_{\sigma}(x)\odot {\Sigma}_{\sigma}(x))\mathbbm{1}_{1\times N})\odot z_i\odot z_i)] \\
&=\frac{1}{n}\sum_{i=1}^{n}\mathbb{E}_{\epsilon_{\sigma}\sim\mathcal{N}(\textbf{0},\textbf{I})}[\sum_{j=1}^{C}(w_j\odot w_j)\cdot ((\epsilon_{\sigma})_j^2\cdot ({\Sigma}_{\sigma}(x))_j^2\cdot ((z_i)_j\odot (z_i)_j))] \\
&=\frac{1}{n}\sum_{i=1}^{n}\mathbb{E}_{\epsilon_{\sigma}\sim\mathcal{N}(\textbf{0},\textbf{I})}[\sum_{j=1}^{C}((\epsilon_{\sigma})_j^2\cdot ({\Sigma}_{\sigma}(x))_j^2\cdot (w_j\cdot (z_i)_j)^2]\\
&=\frac{1}{n}\sum_{i=1}^{n}\sum_{j=1}^{C} ({\Sigma}_{\sigma}(x))_j^2\cdot ||w_j||^2\cdot \langle w_j, (z_i)_j\rangle^2.
\end{split}
\end{equation}
Using the Fubini's theorem for infinite series, we have:
\begin{equation}
\begin{split}
\label{transform6}
&\quad\,\frac{1}{n}\sum_{i=1}^{n}\sum_{j=1}^{C} ({\Sigma}_{\sigma}(x))_j^2\cdot ||w_j||^2\cdot \langle w_j, (z_i)_j\rangle^2 =\frac{1}{n}\sum_{i=1}^{C}(\Sigma_{\sigma})_i^2\cdot\sum_{j=1}^{n}||w_i||_2^{2}\cdot \langle w_i,(\frac{x_j- \mu (x_j)\mathbbm{1}_{1\times N}}{\sigma (x_j)\mathbbm{1}_{1\times N}})_i\rangle^2.
\end{split}    
\end{equation}
Similarly, we can prove:
\begin{equation}
\begin{split}
\label{transform7}
\mathbb{E}_{\epsilon_{\mu}\sim\mathcal{N}(\textbf{0},\textbf{I})}[((\epsilon_{\mu}\odot{\Sigma}_{\mu}(x))\mathbbm{1}_{1\times N})^2 ]=\sum_{i=1}^{C} (\Sigma_{\mu})_i^2\cdot ||w_i||_2^{2} \,.
\end{split}    
\end{equation}
Combining Equation \ref{transform4},\ref{transform5},\ref{transform6} and \ref{transform7}, we see that the final result is proved.
\end{proof}
\end{document}